\documentclass{article}

\usepackage{microtype}
\usepackage{graphicx}
\usepackage{subfigure}
\usepackage{booktabs} 

\usepackage{hyperref}

\usepackage{algorithmic}
\usepackage{amsthm}
\usepackage{amsmath,amssymb} 
\usepackage[T1]{fontenc}
\usepackage{etex}
\usepackage{graphicx}
\usepackage{amsxtra}
\usepackage{amsfonts}
\usepackage{nicefrac}
\usepackage{microtype}
\usepackage{url}
\usepackage{soul}

\usepackage{color}
\usepackage{bbm}
\usepackage{pgfplots}
\usepackage{tikz}

\usepackage{wrapfig}
\usepackage{sidecap}

\usepackage{colortbl}
\newcolumntype{g}{>{\columncolor[gray]{0.85}}c}

\usepackage{epsfig}
\usepackage{stmaryrd} 
\usepackage{mathtools}

\usepackage{multirow}
\usepackage{fixmath}
\usepackage{booktabs}
\usepackage{pbox}
\usepackage{enumerate}
\usepackage{enumitem}

\newcommand{\R}{\mathbb{R}}
\newcommand{\N}{\mathbb{N}}

\newcommand{\la}{\langle}
\newcommand{\ra}{\rangle}

\def\clap#1{\hbox to 0pt{\hss#1\hss}}

\definecolor{fettrot}{RGB}{255,10,10}

\newtheorem{lemma}{Lemma}

\newtheorem{remark}{Remark}

\usepackage{tikz}
\usetikzlibrary{arrows}
    \usetikzlibrary{decorations.pathmorphing}
    \tikzstyle{lifted-edge}=[->,> = latex',blue]
    \tikzstyle{base-edge}=[->,> = latex',black]
	\tikzstyle{cut-edge}=[->,> = latex',dashed]
    \tikzstyle{vertex}=[circle, draw, inner sep=0pt,text height=3mm,text width=3mm,text depth=1.5mm,align=center]
   \tikzset{every picture/.append style={baseline,scale=1.1}}

\makeatletter
\def\nobreakhline{%
\noalign{\ifnum0=`}\fi
\penalty\@M
\futurelet\@let@token\LT@@nobreakhline}
\def\LT@@nobreakhline{%
\ifx\@let@token\hline
\global\let\@gtempa\@gobble
\gdef\LT@sep{\penalty\@M\vskip\doublerulesep}
\else
\global\let\@gtempa\@empty
\gdef\LT@sep{\penalty\@M\vskip-\arrayrulewidth}
\fi
\ifnum0=`{\fi}%
\multispan\LT@cols
\unskip\leaders\hrule\@height\arrayrulewidth\hfill\cr
\noalign{\LT@sep}%
\multispan\LT@cols
\unskip\leaders\hrule\@height\arrayrulewidth\hfill\cr
\noalign{\penalty\@M}%
\@gtempa}
\makeatother
 \usepackage{times}
 \usepackage{placeins}
 \usepackage{longtable}
 \usepackage{thmtools}
 \usepackage{thm-restate}
\usepackage{tikz}
\usetikzlibrary{positioning}
\usetikzlibrary{decorations.markings}

 \setitemize{topsep=2.5pt,parsep=1.5pt,partopsep=1.5pt,leftmargin=0.4cm}
 \mathchardef\mhyphen="2D

\makeatletter
\newcommand\footnoteref[1]{\protected@xdef\@thefnmark{\ref{#1}}\@footnotemark}
\makeatother

\usepackage[accepted]{icml2020}

\icmltitlerunning{Lifted Disjoint Paths with Application in Multiple Object Tracking}

\begin{document}

\twocolumn[
\icmltitle{Lifted Disjoint Paths with Application in Multiple Object Tracking}



\icmlsetsymbol{equal}{*}

\begin{icmlauthorlist}
\icmlauthor{Andrea Hornakova}{equal,mpi}
\icmlauthor{Roberto Henschel}{equal,tnt}
\icmlauthor{Bodo Rosenhahn}{tnt}
\icmlauthor{Paul Swoboda}{mpi}
\end{icmlauthorlist}

\icmlaffiliation{mpi}{Computer Vision and Machine Learning, Max Planck Institute for Informatics, Saarbr\"ucken, Saarland, Germany}
\icmlaffiliation{tnt}{Institut for Image Processing, Leibniz University Hannover, Hannover, Niedersachsen, Germany}

\icmlcorrespondingauthor{Andrea Hornakova}{andrea.hornakova@mpi-inf.mpg.de}
\icmlcorrespondingauthor{Roberto Henschel}{henschel@tnt.uni-hannover.de}

\icmlkeywords{combinatorial optimization, disjoint paths, network flow, polyhedral relaxation, multiple object tracking}

\vskip 0.3in
]



\printAffiliationsAndNotice{\icmlEqualContribution} 

\begin{abstract}
We present an extension to the disjoint paths problem in which additional \emph{lifted} edges are introduced to provide path connectivity priors.
We call the resulting optimization problem the lifted disjoint paths problem.
We show that this problem is NP-hard by reduction from integer multicommodity flow and 3-SAT.
To enable practical global optimization, we propose several classes of linear inequalities that produce a high-quality LP-relaxation.
Additionally, we propose efficient cutting plane algorithms for separating the proposed linear inequalities.
The lifted disjoint path problem is a natural model for multiple object tracking and allows an elegant mathematical formulation for long range temporal interactions.
Lifted edges help to prevent id switches and to re-identify persons. 
Our lifted disjoint paths tracker achieves nearly optimal assignments with respect to input detections. 
As a consequence, it leads on all three main benchmarks of the MOT challenge, improving significantly over state-of-the-art.
\end{abstract}

\section{Introduction}
The disjoint paths problem, a special case of the network flow problem with flows constrained to be binary, is a classical combinatorial optimization problem for which fast combinatorial solvers exist.
It is a natural model for the multiple object tracking problem (MOT) in computer vision~\cite{zhang2008global}.
In the form of the tracking-by-detection paradigm, MOT consists of two steps: First, an object detector is applied to each frame of a video sequence to find the putative locations of all objects appearing in the video.
Then, in the data association step, false positive detections are removed while correct detections are associated to the corresponding identities, thereby forming trajectories.
In this work, we concentrate on the latter task.

While for MOT even very large data association instances can be solved using the disjoint paths formulation, it has been shown that the basic disjoint paths problem alone is not sufficient to provide trajectories of high accuracy.
The main limitation for MOT is the implicit assumption of a first-order Markov chain.
In particular, costs only indicate whether two detections directly follow each other in a track.

Our contribution is three-fold:
First, to overcome the limited expressiveness of disjoint paths, we propose to augment it with lifted edges which take into account long range interactions. 
We call the resulting problem the \emph{lifted disjoint paths problem}, see Section~\ref{sec:problem-formulation}.
We prove the problem to be NP-hard in Section~\ref{sec:complexity}.
Second, we study the optimization problem from a polyhedral perspective, proposing a high-quality linear programming relaxation, see Section~\ref{sec:constraints}.
Separation routines for the proposed constraints are described in Section~\ref{sec:separation}.
Third, we apply the lifted disjoint paths problem to MOT and show that our solver significantly outperforms state-of-the-art trackers on the popular MOT challenge, see Section~\ref{sec:experiments}.

We argue that our model has advantages from the modelling and optimization point of view.
From the modelling standpoint, the lifted disjoint paths problem does not change the set of feasible solutions, but adds more expressive power to it. 
For MOT, this means that the set of feasible solutions, which naturally represent trajectories of objects, is preserved. 
The additional lifted edges represent connectivity priors.
A lifted edge is active if and only if there is an active trajectory between its endpoints in the flow graph. 
For MOT, lifted edges take (dis-)similarity of object detection pairs represented by its endpoints into account. This allows to encourage or penalize an active path between the detections with possibly larger temporal distance. This helps to re-identify the same object and to prevent id-switches between distinct objects within long trajectories.

From the optimization point of view, we study several non-trivial classes of linear inequalities that result in a high-quality relaxation.
The proposed inequalities depend non-trivially on the constraint structure of the underlying disjoint paths problem, see Section~\ref{sec:constraints}.
We show that the polyhedral relaxation we consider is tighter than naively applying known inequalities.
The proposed relaxation enables us to solve MOT problems via a global approach, in contrast to established approaches, which either use heuristics on complex models or global optimization on simpler models that do not exploit long range interaction.
We present, to our knowledge, the first global optimization approach that incorporates long range interaction for MOT.
This has several advantages:
First, our optimization is not trapped in poor local optima or affected by initialization choices and is hence potentially more robust.
Second, improvements in the discriminative power of features used to compute costs for the lifted disjoint paths problem directly correlate to better tracking performance, since no errors are introduced by suboptimal choices during optimization.

Finally, we note that the proposed lifted disjoint path formulation is not inherently tied to MOT and can potentially be applied to further problems not related to MOT. 

Our code is available at \url{https://github.com/AndreaHor/LifT_Solver}.

 \section{Related Work}

\paragraph{Disjoint paths problem.}
The disjoint paths problem 
can be solved with fast combinatorial solvers~\cite{kovacs2015minimum}. 
The shortest paths method for network flow specialized for the disjoint paths problem~\cite{wang2019mussp} performs extremely well in practice.
For the case of the two disjoint paths problem the specialized combinatorial algorithm by Suurballe's~\cite{suurballe1974disjoint} can be used.

There exist several NP-complete extensions to the disjoint paths problem.
The shortest disjoint paths problem with multiple source-sink pairs~\cite{eilam1998disjoint} is NP-complete, as is the more general integer multicommodity flow problem~\cite{EvenMulti}.
The special case of the disjoint paths problem with two distinct source/sink pairs can be solved in polynomial time, however~\cite{tholey2012linear}.

\paragraph{Connectivity priors \& lifted edges.}
For several combinatorial problems, special connectivity inducing edges, which we will call lifted edges for our problem, have been introduced to improve expressiveness of the  base problem.

In the Markov Random Field literature, special connectivity inducing edges were studied from a polyhedral point of view in~\cite{nowozin2010global}.
They were used in image analysis to indicate that two non-adjacent pixels come from the same object and hence they must be part of a contiguously labeled component of the underlying graph.

For multicut (a.k.a.\ correlation clustering), a classical graph decomposition problem, lifted edges have been introduced in~\cite{keuper2015lifted} to model connectivity priors. 
A lifted edge expresses affinity of two nodes to be in the same/different connected component of the graph partition.
Lifted multicut has been used for image and mesh segmentation~\cite{keuper2015lifted}, connectomics~\cite{beier2017multicut} and cell tracking~\cite{rempfler2017efficient}.
A combination of the lifted multicut problem and Markov Random Fields has been proposed in~\cite{levinkov2017joint} with applications in instance-separating semantic segmentation~\cite{kirillov2017instancecut}.
A polyhedral study of lifted multicut was presented in~\cite{hornakova2017analysis}.

Yet, for the above problems, global optimization  has only been reported for small instances.

\paragraph{Disjoint paths for MOT.}
The data association step of MOT has been approached using the disjoint path setup ~\cite{berclaz2011multiple,zhang2008global}, since disjoint paths through a graph naturally model trajectories of multiple objects. 
Extension of the plain disjoint paths problem that disallow certain pairs of detections to occur simultaneously have been used to fuse different object detectors~\cite{chari2015pairwise} and for multi-camera MOT~\cite{hofmann2013hypergraphs,leal2012branch}. 
The drawback of these approaches is that they cannot integrate long range information, in contrast to our proposed formulation.

\paragraph{Other combinatorial approaches to MOT.}

The minimum cost arborescence problem, an extension of minimum spanning tree to directed graphs, has been used for MOT in~\cite{henschel2014efficient}.
In~\cite{keuper2016multi,keuper2018motion,kumar2014multiple,ristani2014tracking,tang2015subgraph,tang2016multi} the multicut problem has been used for MOT and in~\cite{babaee2018multiple,tang2017multiple} additionally lifted edges have been used to better model long range temporal interactions.
The maximum clique problem, which corresponds to multicut with complete graphs has been applied for MOT in~\cite{zamir2012gmcp,dehghan2015gmmcp}.
Maximum independent set, which corresponds to maximum clique on the complement graph, has been used for MOT in~\cite{brendel2011multiobject}.
The multigraph-matching problem, a generalization of the graph matching problem, has been applied to MOT in~\cite{hu2019dual}.
Consistency of individual matched detections is ensured by cycle-consistency constraints coming from the multi-graph matching.
The works~\cite{Henschel_2018_CVPR_Workshops,henschel2016tracking} reformulate tracking multiple objects with long temporal interactions as a binary quadratic program. If the problem size is small, the optimization problem can be solved optimally by reformulating it to an equivalent binary linear program \cite{henschel2019simultaneous,von2018recovering}. For large instances, an approximation is necessary.  To this end, a specialized non-convex Frank-Wolfe method can be used \cite{Henschel_2018_CVPR_Workshops}.
Common to the above state of the art trackers is that they either employ heuristic solvers or are limited in the integration of long range information, in contrast to our work.

\paragraph{Contribution w.r.t.\ existing combinatorial approaches.}
It is widely acknowledged that one crucial ingredient for obtaining high-quality MOT results is to incorporate long range temporal information to re-identify detections and prevent id-switches.
However, from a theoretical perspective, we believe that long range information has not yet been incorporated satisfactorily in optimization formulations for the data association step in MOT.

In comparison to lifted multicut for MOT, we argue that from the modelling point of view, network flow has advantages.
In multicut, clusters can be arbitrary, while in MOT, tracks are clusters that may not contain multiple detection hypotheses of distinct objects at the same time point.
This exclusion constraint must be enforced in multicut explicitly via soft constraints, while the disjoint paths substructure automatically takes care of it.
On the other hand, the lifted multicut approach~\cite{tang2017multiple} has used the possibility to cluster multiple detections in one time frame.
This directly incorporates non-maxima suppression in the optimization, which however increases computational complexity.

From a mathematical perspective, naively using polyhedral results from multicut is also not satisfactory.
Specifically, one could naively obtain a polyhedral relaxation for the lifted disjoint paths problem by reusing the known polyhedral structure of lifted multicut~\cite{hornakova2017analysis} and additionally adding network flow constraints for the disjoint paths substructure.
However, this would give a suboptimal polyhedral relaxation. 
We show in Section~\ref{sec:constraints} that the underlying structure of the disjoint paths problem can be used to derive new and tighter constraints for lifted edges.
This enables us to use a global optimization approach for MOT.
To our knowledge, our work is the first one to combine global optimization with long range interactions for MOT.

In comparison to works that propose non-convex algorithms or other heuristics for incorporating long range temporal edges~\cite{Henschel_2018_CVPR_Workshops,hu2019dual,zamir2012gmcp,dehghan2015gmmcp} our approach yields a more principled approach and globally optimal optimization solutions via LP-based branch and bound algorithms.

 \section{Problem Formulation}
\label{sec:problem-formulation}
Below we recapitulate the disjoint paths problem and extend it by defining lifted edges.
We discuss how the lifted disjoint paths problem can naturally model MOT.
Proofs for statements in all subsequent sections can be found in the Appendix, Section~\ref{sec:appendix}.

\paragraph{Flow network and lifted graph.}
Consider two directed acyclic graphs $G = (V,E)$ and $G'=(V',E')$ where $V'=V\backslash \{s,t\}$.
The graph $G=(V,E)$ represents the \emph{flow network} and we denote by $G'$ the lifted graph. 
The two special nodes $s$ and $t$ of $G$ denote source and sink node respectively. 
We further assume that every node in $V$ is reachable from $s$, and $t$ can be reached from it.

We define the set of paths starting at $v$ and ending in $w$ as
\begin{equation}
    vw\mhyphen\text{paths}(G) = \left\{ (v_1 v_2,\ldots,v_{l-1} v_l) : \begin{array}{c}v_i v_{i+1} \in E,\\ v_1 = v, v_l = w \end{array} \right\}\,.
\end{equation}
For a $vw\mhyphen$path $P$ we denote its edge set as $P_E$ and its node set as $P_V$.

The flow variables in $G$ are denoted by $y \in \{0,1\}^E$ for edges and $x \in \{0,1\}^V$ for nodes. Allowing only 0/1 values of vertex variables reflects the requirement of vertex disjoint paths.
Variables on the lifted edges $E'$ are denoted by $y'\in\{0,1\}^{E'}$.
Here, $y'_{vw}=1$ means that nodes $v$ and $w$ are connected via the flow $y$ in $G$.
Formally, 
\begin{equation}
\label{eq:lifted-edge-def}
y'_{vw} = 1 \Leftrightarrow \exists P\in vw\mhyphen\text{paths}(G) \text{  s.t.\  } \forall ij\in P_E: y_{ij}=1 \,.
\end{equation}

\paragraph{Optimization problem.}
Given edge costs $c \in \R^E$, node cost $\omega \in \R^V$ in flow network $G$ and edge cost $c' \in \R^{E'}$ for the lifted graph $G'$ we define the lifted disjoint paths problem as
\begin{equation}
\label{eq:lifted-disjoint-paths-problem}
\begin{array}{rl}
    \min\limits_{\substack{y \in \{0,1\}^E, y' \in \{0,1\}^{E'},\\ x \in \{0,1\}^{V}}} & \la c, y \ra + \la c',y' \ra + \la \omega, x\ra \\
    \text{s.t.} & y \text{ node-disjoint } s,t \text{-flow in } G, \\
                &x \text{ flow through nodes of }G\\
                & y, y' \text{ feasible according to } \eqref{eq:lifted-edge-def}
    \end{array}
\end{equation}

In Section \ref{sec:constraints}, we present an ILP formulation of~\eqref{eq:lifted-disjoint-paths-problem} by proposing
several linear inequalities that lead to a high-quality linear relaxation.

\paragraph{Graph construction for multiple object tracking.}
We argue that the lifted disjoint paths problem is an appropriate way of modelling the data association problem for MOT.
In MOT, an unknown number of objects needs to be tracked across a video sequence.
This problem can be naturally formalized by a graph $G=(V,E)$ where its node set $V$ represents either object detections or tracklets of objects.
If $V$ represents object detections, we can express it as follows: $V =  s\cup V_1 \cup \ldots \cup V_T \cup t$, where $T$ is the number of frames and $V_i$ denotes the object detections in time $i$.
We introduce edges between adjacent time frames.
An active flow on such an edge denotes correspondences of the same object.
We also introduce skip edges  between time frames that are farther apart.
An active flow on a skip edge also denotes correspondences between the same object that, in contrast, may have been been occluded or not detected in intermediate time frames.
This classical network flow formulation has been commonly used for MOT~\cite{zhang2008global}.

On top of the underlying flow formulation for MOT, we usually want to express that two detections belong to the same object connected by a possibly longer track with multiple detections in between.
For that purpose, lifted edges with negative costs can be used.
We say in such a case that an active lifted edge re-identifies two detections~\cite{tang2017multiple}.
If two detections with larger temporal distance should not be part of the same track, a positive valued lifted edge can be used.
In this case the lifted edge is used to prevent id-switches.

 \section{Constraints}
\label{sec:constraints}
Below, we will first introduce constraints that give an integer linear program (ILP) of the lifted disjoint paths problem~\eqref{eq:lifted-disjoint-paths-problem}.
The corresponding linear programming (LP) relaxation can be strenghtened by additional constraints that we present subsequently.

Many constraints considered below will rely on whether a node $w$ is reachable from another node $v$ in the flow network. We define to this end the \emph{reachability relation} $\mathcal{R} \subset V^2$ via
\begin{equation}
vw \in \mathcal{R} \Leftrightarrow vw\mhyphen\text{paths}(G) \neq \emptyset \,.
\end{equation}
In the special case of $v = w$, we also allow empty paths, which means $\forall v\in V: vv \in \mathcal{R}$.
This makes relation $\mathcal{R}$ reflexive.

\paragraph{Flow conservation constraints.}
The flow variables $y$ obey, as in classical network flow problems~\cite{ahuja1988network}, the flow conservation constraints
\begin{align}
  \label{eq:flow-conservation} \forall v\in V\setminus\{s,t\}: \sum\limits_{u:uv\in E} {y}_{uv}=\sum\limits_{w:vw\in E} {y}_{vw}={x}_{v}\,.
\end{align}

\paragraph{Constraining lifted edges.} 
     All the following constraints restrict values of lifted edge variables $y'_{vw}$ in order to ensure that they satisfy \eqref{eq:lifted-edge-def}. Despite their sometimes complex form, they always obey the two basic principles:
\begin{itemize}
 \item If there is flow in $G$ going from vertex $v$ to vertex $w$, then $y'_{vw}=1$.
 The constraints of this form are \eqref{eq:path-inequalities}, \eqref{eq:lifted-path-inequalities}.
 \item If there is a $vw$-cut in $G$ with all edges labeled by zero (i.e. no flow passes through this cut), then $y'_{vw}=0$.
 We will mainly look at cuts that are induced by paths, i.e.\ edges that separate a path from the rest of the graph. 
 The paths of interest will either originate at $v$ or end at $w$.
 The constraints of this form are \eqref{eq:vwcut1}, \eqref{eq:vwcut2}, \eqref{eq:path-induced-cut-inequality}, \eqref{eq:lifted-path-induced-cut-inequality}, \eqref{eq:lifted-path-induced-cut-inequality2}.
\end{itemize}

\paragraph{Single node cut inequalities.}
Given a lifted edge $vw \in E'$, if there is no flow going from vertex $v$ which can potentially go to vertex $w$, then $y'_{vw}=0$. 
Formally,
\begin{equation}
\label{eq:vwcut1}  y'_{vw} \leq \sum_{\substack{u:\ vu \in E,\\ uw \in \mathcal{R}}}y_{vu}\,.
\end{equation}
Similarly, if there is no flow going to $w$ that can originate from vertex $v$, then $y'_{vw}=0$.
Formally,
\begin{equation}
\label{eq:vwcut2}  y'_{vw} \leq \sum_{\substack{u: uw\in E,\\ vu \in \mathcal{R}}} y_{uw}\,.
\end{equation}

The number of constraints of the above type~\eqref{eq:flow-conservation} is linear in the number of vertices, while~\eqref{eq:vwcut1} and~\eqref{eq:vwcut2} are linear in the number of lifted edges.
Hence we add them into our initial constraint set during optimization.

\paragraph{Path inequalities.}
For lifted edge $y'_{vw}$ it holds that if there is
a flow in $G$ going from $v$ to $w$ along a path $P$,
then $y'_{vw}=1$. This constraint can be expressed by the following set of inequalities:
\begin{align}
\label{eq:path-inequalities}
  \nonumber   \forall vw\in E'\ &\forall P\in vw\mhyphen\text{paths}(G):\\
  y'_{vw}&\geq \sum_{vj:j\in P_V}y_{vj}- \sum_{i\in P_V\setminus \{v,w\}}\sum_{k\notin P_V}y_{ik} 
\end{align}
Here the first sum expresses the flow going from $v$ to any vertex of path $P$. The second sum is the flow leaving path vertices $P_V$ before reaching $w$.
In other words, if flow does not leave $P_V$, edge $y'_{vw}$ must be active.
Note that inequality~\eqref{eq:path-inequalities} implicitly enforces $y'_{vw}$ to be active if any path $vw$-path $\tilde{P}$ with $\tilde{P}_V \subset P_V$ is active.

\begin{remark}
For the multicut problem, there exist path inequalities that enforce path properties in an analogous way.
While the multicut path inequalities would yield the same set of feasible integral points, the resulting polyhedral relaxation would be weaker, see Proposition~\ref{prop:multicut-path-inequalities weaker} in the Appendix.
\end{remark}

\paragraph{Path-induced cut inequalities.}
The path-induced cut inequalities generalize the single node cut inequalities~\eqref{eq:vwcut1} and~\eqref{eq:vwcut2} by allowing cuts induced by paths.

Let a lifted edge $vw \in E'$, a node $u$ from which $w$ is reachable and a $vu$-path $P$ be given.
Consider the cut given by edges $ik$ with $i \in P_V$ and $k \notin P_V$ but such that $w$ is reachable from $k$.
If the flow does not take any edge of this cut, then $y'_{vw} = 0$.
Formally,
\begin{gather}
  \nonumber   \forall vw\in E'\ \forall P\in vu\mhyphen\text{paths}(G)\text{ s.t. }\ uw \in \mathcal{R}  \wedge u \neq w:\\
 \label{eq:path-induced-cut-inequality}y'_{vw}\leq \sum_{i\in P_V} \sum_{\substack{k\notin P_V,\\ kw \in \mathcal{R}}}y_{ik} \,.
\end{gather}

\paragraph{Lifted inequalities.}
The path inequalities~\eqref{eq:path-inequalities} and the path-induced cut inequalities~\eqref{eq:path-induced-cut-inequality} only consider base edges on their right hand sides.
We can generalize both~\eqref{eq:path-inequalities} and~\eqref{eq:path-induced-cut-inequality} by including lifted edges in the paths as well.
Conceptually, using lifted edges allows to represent all possible paths between their endpoints, which enables to formulate tighter inequalities, see Propositions~\ref{prop:lifted-path-strictly-better} and~\ref{prop:lifted-path-induced-cut-inequality-strictly-better}.

To that end consider the multigraph $G\cup G':=(V,E\cup E')$.
For any edge $ij \in E\cap E'$ we always distinguish whether $ij \in E$ or $ij \in E'$.
For $P\in vw\mhyphen\text{paths}(G\cup G')$, we denote by $P_{E}$ and $P_{E'}$ edges of the path $P$ in $E$ and $E'$ respectively. 
We require $P_{E}\cap P_{E'}=\emptyset$. 

\paragraph{Lifted path inequalities.}
We generalize the path inequalities~\eqref{eq:path-inequalities}.
Now the $vw$-path $P$ may contain both edges in $E$ and $E'$.
Whenever a lifted edge $y'_{ij}$ in the third sum in~\eqref{eq:lifted-path-inequalities} is one, two cases can occur:
(i)~Flow goes out of $P$ (uses vertices not in $P_V$) but reenters it again later.
Then a base edge variable $y_{ik}$ will be one in the second sum in~\eqref{eq:lifted-path-inequalities} and the values of $y'_{ij}$ and $y_{ik}$ cancel out.
(ii)~A base edge $ij \in E \cap E'$ parallel to the lifted edge is active.
Then the variable $y_{ij}$ in the fourth sum in~\eqref{eq:lifted-path-inequalities} cancels out $y'_{ij}$.
The lifted path inequality becomes
\begin{align}
  \nonumber   \forall vw\in E'\ &\forall P\in vw\mhyphen\text{paths}(G\cup G'):\\
\nonumber  y'_{vw}\geq &\sum_{j\in P_V}y_{vj} - \sum_{i\in P_V\setminus \{v,w\}}\sum_{k\notin P_V} y_{ik}\\
& + \sum_{ij\in P_{E'}} y'_{ij}-\sum_{ij\in P_{E'}\cap E} y_{ij}\,.\label{eq:lifted-path-inequalities}
\end{align}
Whenever the path in~\eqref{eq:lifted-path-inequalities} consists only of base edges $P_E$, the resulting inequality becomes a path inequality~\eqref{eq:path-inequalities}.

\begin{restatable}{prop}{liftedstrictlybetter}
\label{prop:lifted-path-strictly-better}
The lifted path inequalities~\eqref{eq:lifted-path-inequalities} provide a strictly better relaxation than the path inequalities~\eqref{eq:path-inequalities}.
\end{restatable}

\paragraph{Lifted path-induced cut inequalities.}
We generalize the path-induced cut inequalities~\eqref{eq:path-induced-cut-inequality}.
Let a lifted edge $vw \in E'$ and a $vu\mhyphen$path $P$ in $G \cup G'$ be given.
In contrast to the basic version~\eqref{eq:path-induced-cut-inequality}, a lifted edge $ij \in P_{E'}$ can be taken.
This can occur in two cases: Either the flow leaves $P_V$ via a base edge $ik$, $k\notin P_V$ or a base edge $ij \in E \cap E'$ parallel to the lifted edge is taken. Both cases are accounted for by terms in the first and the third sum in~\eqref{eq:lifted-path-induced-cut-inequality} below.

\begin{gather}
  \nonumber   \forall vw\in E'\ \forall P\in vu\mhyphen\text{paths}(G \cup G')\text{ s.t. }\ uw \in \mathcal{R} \wedge u \neq w: \\
   y'_{vw}\leq \sum_{i\in P_V} \sum_{\substack{k\notin P_V,\\ kw \in \mathcal{R}}}y_{ik}-\sum_{ij\in P_{E'}}y'_{ij}
 +\sum_{ij\in P_{E'}\cap E}y_{ij} 
\label{eq:lifted-path-induced-cut-inequality}
\end{gather}

Assume that the last node $u$ of path $P$ is connected via a lifted edge with $w$.
Then we can strengthen~\eqref{eq:lifted-path-induced-cut-inequality} by replacing the sum of base edges outgoing from $u$ by $y'_{uw}$.
\begin{gather}
 \nonumber \forall vw\in E'\ \forall P\in vu\mhyphen\text{paths }(G\cup G')\ \text{ s.t. } uw\in E':\\
 \nonumber y'_{vw}\leq \sum_{i\in P_V\setminus u} \ \sum_{\substack{k\notin P_V, \\ kw \in \mathcal{R}}} y_{ik}  -\sum_{ij\in P_{E'}}y'_{ij}\\
   +\sum_{ij\in P_{E'}\cap E}y_{ij}+y'_{uw} 
  \label{eq:lifted-path-induced-cut-inequality2} 
\end{gather}

\begin{restatable}{prop}{liftedpathinducedcutbetter}
\label{prop:lifted-path-induced-cut-inequality-strictly-better}
The lifted path-induced cut inequalities~\eqref{eq:lifted-path-induced-cut-inequality} define a strictly tighter relaxation than the path-induced cut inequalities~\eqref{eq:path-induced-cut-inequality}.

Furthermore the lifted path-induced cut inequalities~\eqref{eq:lifted-path-induced-cut-inequality} and~\eqref{eq:lifted-path-induced-cut-inequality2} define a strictly better relaxation than~\eqref{eq:lifted-path-induced-cut-inequality} alone.
\end{restatable}

\paragraph{Symmetric cut inequalities.}
Inequalities \eqref{eq:vwcut2} provide a symmetric counterpart to inequalities \eqref{eq:vwcut1}. 
We can also formulate symmetric counterparts to  inequalities \eqref{eq:path-induced-cut-inequality}, \eqref{eq:lifted-path-induced-cut-inequality} and \eqref{eq:lifted-path-induced-cut-inequality2} by swapping the role of $v$ and $w$.
 All constraints \eqref{eq:path-induced-cut-inequality}, \eqref{eq:lifted-path-induced-cut-inequality} and \eqref{eq:lifted-path-induced-cut-inequality2} concentrate on paths originating in $v$. 
The symmetric inequalities are obtained by studying all paths ending in $w$. 
These symmetric inequalities are described in Appendix Section~\ref{sec:symmetric-inequalities}.
Relations analogous to those described in Proposition~\ref{prop:lifted-path-induced-cut-inequality-strictly-better} hold for the symmetric counterparts as well.
The symmetric inequalities also strengthen the relaxation strictly.
For the exact statements, see propositions in Appendix Section~\ref{sec:symmetric-inequalities}.

 \section{Separation}
\label{sec:separation}
We solve the lifted disjoint paths problem~\eqref{eq:lifted-disjoint-paths-problem} with the state of the art integer linear program solver Gurobi~\cite{gurobi}.
Since there are exponentially many constraints of the form~\eqref{eq:path-inequalities}, \eqref{eq:path-induced-cut-inequality}, \eqref{eq:lifted-path-inequalities}, \eqref{eq:lifted-path-induced-cut-inequality} and~\eqref{eq:lifted-path-induced-cut-inequality2}, we do not add them initially.
Instead, we start with constraints~\eqref{eq:flow-conservation}, \eqref{eq:vwcut1} and \eqref{eq:vwcut2} and find the optimal integer solution.
In the separation procedures described below we check if any of the advanced constraints are violated and add those that are to the active constraint set.
We resolve the tightened problem and iterate until we have found a feasible solution to the overall problem~\eqref{eq:lifted-disjoint-paths-problem}.

Algorithms \ref{alg:lifted1-separation} and \ref{alg:lifted2-separation} describe the separation procedures for adding lifted path constraints \eqref{eq:lifted-path-inequalities}, and lifted path-induced cut constraints \eqref{eq:lifted-path-induced-cut-inequality} and \eqref{eq:lifted-path-induced-cut-inequality2}.
Since path constraints~\eqref{eq:path-inequalities} and path-induced cut inequalities~\eqref{eq:path-induced-cut-inequality} are special cases of those above, they are also accounted for.

\paragraph{Separation for path inequalities.}
Algorithm~\ref{alg:lifted1-separation} iterates over all active $st$-paths.
For every path $P^1$, labels of all lifted edges connecting two vertices in $P^1_{V}$ are inspected.
If the lifted edge variable is zero, Algorithm~\ref{alg:lifted1-separation} will extract a path in $G \cup G'$ connecting the endpoints and add the resulting lifted path inequality~\eqref{eq:lifted-path-inequalities} to the active constraint set.
\begin{algorithm}[!ht]
    \caption{Separation for lifted path inequalities~\eqref{eq:lifted-path-inequalities}}
    \label{alg:lifted1-separation}
    \begin{algorithmic}%
        \STATE Define $E^1 = \{ e \in E: y_e = 1\}$, $G^1 = (V, E^1)$
        \FORALL {$P^1\in st$-paths $(G^1)$}
          \FORALL{$y'_{vw}=0: v\in P^1_V \wedge w\in P^1_V$}          
          \STATE $P:=$ \textrm{Extract\_path($P^1,v,w$)}
          \STATE Add constr.~\eqref{eq:lifted-path-inequalities} for $y'_{vw}$ with  $P$.
         \ENDFOR
        \ENDFOR
    \end{algorithmic}
\end{algorithm}

\paragraph{Separation for path-induced cut inequalities.}
Algorithm~\ref{alg:lifted2-separation} iterates over all active $st$-paths.
For every path $P^1$, lifted edges that start in $P^1_V$ but do not end in $P^1_V$ are inspected.
If their label is one, Algorithm~\ref{alg:lifted2-separation} will extract a subpath of $P^1$ for either~\eqref{eq:lifted-path-induced-cut-inequality2} or~\eqref{eq:lifted-path-induced-cut-inequality} and add the respective inequality to the active constraint set.
 \begin{algorithm}[!ht]
    \caption{Separation for lifted path-induced cut inequalities~\eqref{eq:lifted-path-induced-cut-inequality} and \eqref{eq:lifted-path-induced-cut-inequality2}}
    \label{alg:lifted2-separation}
    \begin{algorithmic}
       \STATE Define $E^1 = \{ e \in E: y_e = 1\}$, $G^1 = (V, E^1)$
        \FORALL{$P^1\in st\mhyphen$paths $(G^1)$}
          \FORALL{$y'_{vw}=1: v\in P^1_V \wedge w\notin P^1_V$}
          \IF{$\exists u\in P^1_V:  y'_{uw}=0\wedge vu\in\mathcal{R}$}
          \STATE $P:=$ \textrm{Extract\_path($P^1,v,u$)}
          \STATE Add constr.~\eqref{eq:lifted-path-induced-cut-inequality2} for $y'_{vw}$ with  $P$.
         \ELSE
         \STATE $u:=$ last vertex of $P^1$ such that $uw \in \mathcal{R}$
         \STATE $P:=$ \textrm{Extract\_path ($P_1,v,u$)}
          \STATE Add constr.~\eqref{eq:lifted-path-induced-cut-inequality} for $y'_{vw}$ with  $P$.
         \ENDIF      
         \ENDFOR
        \ENDFOR
    \end{algorithmic}
\end{algorithm}
\paragraph{Complexity of separation.}
Both Algorithms \ref{alg:lifted1-separation} and \ref{alg:lifted2-separation} can be implemented efficiently such that they are linear in $|E^1|$ (i.e. in the number of active edges of graph $G$).
In our implementation, we traverse all active $st$-paths from the end to the beginning and directly store correctly labelled lifted edges that originate on the already processed subpaths. These lifted edges can be used later as edges in $P_{E'}$ in \eqref{eq:lifted-path-inequalities}-\eqref{eq:lifted-path-induced-cut-inequality2} or as $y'_{uw}=0$ in \eqref{eq:lifted-path-induced-cut-inequality2}. 
\begin{algorithm}[!ht]
    \caption{Extract\_path($P^1,v,w$)}
    \label{alg:extract-path}
    \begin{algorithmic}
        \STATE $P':=$ $vw$-subpath of $P^1$, $P:=\emptyset$
        \FOR{$j \in P'_V$ from end of path to beginning} 
        \IF{$\exists$ edge $ij \in E'$, $i \in P'_V, y'_{ij} = 1$}
        \STATE Add $ij$ to $P_{E'}$, skip to node $i \in P'_V$
        \ELSE 
        \STATE Add $ij$ from $P'$ to $P_E$
        \ENDIF
        \ENDFOR
        \OUTPUT{$P=P_E\cup P_{E'}$}
    \end{algorithmic}
\end{algorithm}


 \section{Complexity}
\label{sec:complexity}

Below, we show that the lifted disjoint paths problem~\eqref{eq:lifted-disjoint-paths-problem} is NP-hard. The following Theorems state that even its restricted versions  using only negative or only positive lifted edges are NP-hard.
The proofs use reductions from two known NP-complete problems. Theorem~\ref{thm:multicommodity-flow-reduction} is proven by reduction from integer multicommodity flow~\cite{EvenMulti} and Theorem~\ref{thm:3sat-flow-reduction}  by reduction from 3-SAT~\cite{cook1971complexity}.

\begin{restatable}{thm}{multicommodityflowreduction}
\label{thm:multicommodity-flow-reduction}
Lifted disjoint paths problem~\eqref{eq:lifted-disjoint-paths-problem} with negative lifted edges only is NP-hard.
\end{restatable}

\begin{restatable}{thm}{threesatreduction}
\label{thm:3sat-flow-reduction}
Lifted disjoint paths problem~\eqref{eq:lifted-disjoint-paths-problem} with positive lifted edges only is NP-hard.
\end{restatable}

 \section{Experiments}
\label{sec:experiments}

We conduct several experiments on MOT showing the merit of using lifted disjoint paths for the tracking problem.
Below, we describe our problem construction, cost learning for base and lifted edges, preprocessing and post-processing steps and report resulting performance. More details about our experiments are provided in Appendix, Section \ref{sec:appendix}.
\subsection{Graph Construction.}

\paragraph{Two-step procedure.}
Due to the computational complexity of the problem, we cannot solve entire video sequences straightforwardly. 
In order to make the problem tractable, we apply the following two-step procedure.
In the first step, the solver is applied on graphs over person detections but only for small time intervals consisting of a few dozen video frames.  
The tracks resulting from the first step are used for extracting tracklets. 
In the second step, the solver is applied on newly created graphs $G$ and $G'$ where vertices correspond to the obtained tracklets.
Edges and edge costs between tracklets are obtained by aggregating original edges resp.\ edge costs between person detections.
The tracks resulting from the second step may be suboptimal with respect to the original objective function defined over person detections.
Therefore, we identify points where splitting a track leads to an improvement of the original objective value and extract new tracklets from the divided tracks. 
Multiple iterations of the second step are performed until no improving split points are found in the output tracks.
This two-step procedure improves the objective w.r.t.\ the original objective~\eqref{eq:lifted-disjoint-paths-problem} in every iteration.
Since there are only finitely many trackings, the procedure terminates finitely.
In practice, only a few iterations are necessary.

\paragraph{Graph sparsification. }For our experiments, we use edges between detections up to $2 \mathrm{sec}$ temporal distance. These long range edges cause high computational complexity for the first step. In order to reduce it, we apply sparsification on both base and lifted graphs. For the base edges, we select for every $v \in V'$ its $K$ nearest (lowest-cost) neighbors from every subsequent time frame within an allowed time gap. Lifted edges with costs close to zero are not included, since they are not discriminative. Lifted edges connecting detections with high time gap are included more sparsely than lifted edges having lower time gaps. We use dense graphs in the second step. 

\paragraph{Costs.}
Initially, in the first step, we set $\omega_v=0$ for all vertices $v\in V$.
For the second step, where $V$ represents tracklets, $\omega_v$ is set to the cost of outputting tracklet $v$ as a final trajectory. 
Specifically, $\omega_v$ is the sum of costs of base edges between consecutive detections in the tracklet and the cost of lifted edges between all pairs of detections contained in the tracklet.
The cost of a base edge between two tracklets is given by the cost of the original base edge connecting the last detection in the first tracklet with the first detection in the subsequent tracklet.
The cost of a lifted edge between two tracklets is obtained by summing up the costs of original lifted edges between detections contained in the tracklets.
This ensures that the costs of the tracklet solution corresponds to the costs of the original problem. We set cost of all edges from the source node $s$ and to the sink node $t$ to zero. Setting of detection costs and in/out costs to zero reduces the number of hyperparameters that  usually needs to be incorporated by other methods. Moreover, our method does not include temporal decay of edge costs since the formulation directly prefers short range base edges over the long range ones.

\subsection{Preprocessing and Post-processing}
\label{sec:pre_post_processing}
As is common for tracking by detection, we perform pre- and  postprecessing to compensate for detector inaccuracies.

\textbf{Input filtering.} Given a set of input detections derived from a detector, we follow the approach of~\cite{bergmann2019tracking}, a leading tracker for the MOT challenge, to reject false positive detections and to correct misaligned ones.
For this, each input detection is send through the regression and classification part of their detector.
In more detail, all tracking parts involved in the tracker Tracktor \cite{bergmann2019tracking} are deactivated, such that it only reshapes and eventually rejects input  detections, without assigning labels to them.
Input detections are rejected if  Tracktor's detector outputs a confidence score $ \sigma_{\mathrm{active}} \leq 0.5$.

Tracktor also applies a non-maxima-surpression on the reshaped input detections, where we use the threshold $\lambda_{\mathrm{new}} = 0.6$.

\textbf{Inter- and extrapolation.} 
Even if all input detections have been assigned to the correct identities by our solver, there might still be  missing detections in case that a person has not been detected in some frames. We recover missing detections within the time range of a trajectory, which we denote as interpolation. Further, we extend a trajectory in forward and backward directions, which we denote as extrapolation.
To this end, we follow~\cite{bergmann2019tracking} and apply their object detector to recover missing positions based on the visual information at the last known position. 
Finally, for sequences filmed from a static camera, we perform linear interpolation on the remaining gaps. 
These sequences can be automatically detected using DeepMatching on the regions outside detection boxes.

To demonstrate the performance using traditional post-processing, we also evaluate our tracker using only linear interpolation as post-processing in all sequences.

\subsection{Cost Learning}
\label{sec:cost_learning}
Costs for base edges $E$ and lifted edges $E'$ are computed equally, since they both indicate whether two detections are from the same object or not.
For an edge $e=vw$, we denote with $d_{\mathrm{wi}}(v)$ the detection width corresponding to node $v$.

\paragraph{Visual cues.}
We exploit two different appearance features:
Given two detections, the \emph{re-identification} descriptor utilizes global appearance statistics, 
while the \emph{deep-matching} descriptor relies on fine-grained pixel-wise correspondences.

We employ the state-of-the-art re-identification network~\cite{zheng2019joint} and train it on MOT17 train set~\cite{MOT16} together with additional re-identification datasets~\cite{zheng2015scalable,wei2018person,ristani2016MTMC}. The obtained feature value $\mathfrak{f}_{\mathrm{re\mhyphen id}}(e) \in [-1,1]$ is modified in order to better reflect the uncertainty of a connection.
We truncate values smaller $0$ (corresponding to improbable connections) and re-scale the rest. First, we normalize scores between each detection $v$ and all detections in every time frame $V_j$ through the score of the most probable connecting edge $vw$. Second, all other connections than $vw$ are downscaled.

Our second visual cue utilizes DeepMatching (DM)~\cite{weinzaepfel:hal-00873592} to establishes pixelwise correspondences between two images. It thus serves as a reliable tracking feature \cite{tang2016multi,Henschel_2018_CVPR_Workshops,henschel2019multiple}.

We apply DM between boxes in two images and compute the DM intersection over union~\cite{tang2016multi,Henschel_2018_CVPR_Workshops}  w.r.t.\ the whole detection boxes and on five subboxes (left/right, upper/middle/lower part).
In addition,  we measure for all points in a given subbox whether their matched endpoints are in the corresponding subbox again or not. This gives two additional error measures for deviation in $x$ and $y$-directions.
Thus, in total we obtain a feature vector $\mathfrak{f}_{DM}(e) \in [0,1]^8$.
In order to assess the reliability of DM features, density of matching points is computed in each box and its subboxes. The smaller value is chosen for each box pair. This results in feature $\rho \in [0,1]^6$.

\paragraph{Motion constraints.}
We penalize for improbable motions by comparing the maximal displacement of DM endpoints within the sequence with the displacements of detection boxes. Assignment hypotheses of pairs of boxes representing improbable motions are penalized with a large cost.

\paragraph{Spatio-temporal cues.}
Our spatio-temporal cues utilize a simple motion compensation by computing the median DM displacement between correspondences of the background. 

We assume a linear motion model, similar to~\cite{ristani2018features} and penalize deviations of detections from the estimated motion trajectory.
This enforces  spatio-temporally consistency of detections within one trajectory.
Furthermore, we penalize improbable large person movements by relating velocities (in pixels per seconds) in horizontal direction to box width:  $\mathfrak{f}_{\mathrm{trans}}(e)= \log({\mathfrak{v}_{x}(e)}/\min\{d_{\mathrm{wi}}(v),d_{\mathrm{wi}}(w)\})$.

\paragraph{Fusion of input features.}
We construct a neural network consisting of fully connected layers, batch normalization and relu units taking the above described features and time differences as input and outputting scores for assignment hypotheses.
The final layer uses a sigmoid activation function for producing a score in $[0,1]$.
We refer to the supplemental material for the exact structure of the neural network and details about the training procedure. 

\subsection{Experiment Setup}
In order to assess the suitability of the proposed lifted disjoint paths formulation for MOT, we conduct extensive experiments on three challenging benchmarks: MOT15~\cite{MOTChallenge2015}, MOT16 and MOT17~\cite{MOT16}, resulting in 39 test sequences. The sequences are filmed from static and moving cameras. While MOT16 and MOT17 share the same sequences, MOT17 provides three different detectors in order to study the dependence of the tracking quality on the input detections. 
We perform analysis and parameter tuning for our tracker on the MOT17 train set, even when our tracker is applied to the MOT15 sequences to ensure that our tracker is not prone to overfitting. 
We follow the MOT challenge protocol and use the detections provided by the respective benchmarks. 
All experiments on the training set are evaluated using a leave-one-out cross-validation. This includes all of our training procedures,  in particular also the training of the re-identification  network. 

To measure the tracking quality, the multiple object tracking accuracy (MOTA) \cite{bernardin2008evaluating} and the IDF1 metric \cite{ristani2016MTMC} are regarded as the most meaningful ones. The first incorporates the number of false negatives (FN), false positives (FP) and identity switches (IDS), thereby focusing on the coverage of persons. The latter assesses the consistency w.r.t.\ identities. Further tracking metrics (MT, ML) are defined in \cite{li2009learning}.

\subsection{Benefit of Long Range Edges}
\label{sec:exp_long_range_edges}
We investigate the importance of using long range information for MOT. To this end, we apply our proposed tracker on the MOT17 training sequence with varying maximal time gap, for which base and lifted edges are created between nodes. In order to assess the influence of the time gap on the tracking quality, we measure the \textit{assignment} quality in terms of the MOTA and IDF1 metrics, without performing any inter- or extrapolation.
To assess how well the \textit{assignment part} is solved by our tracker, we compute the maximum achievable metrics given the filtered input detections and admissible assignment hypotheses within maximal time gaps. A detailed description of how we obtain the optimal assignments are given in the appendix in Section \ref{sec:optimal_assignment_definition}.
From the result in Table~\ref{tab:tracking_performance_over_time}, we see essentialy constant MOTA scores. 
This is due the fact that selecting correct connections does not change MOTA significantly except after inter- and extrapolation (which we have excluded in Table~\ref{tab:tracking_performance_over_time}). 
However, we see a significant improvement in the IDF1 score, which directly penalizes wrong connections. Here, long range edges help greatly.
Moreover, both metrics, ID precision and ID recall, clearly increase with increasing time gap. This shows that improvements by incorporating more temporal information come from using longer skip edges (impact on IDR) but most importantly, precision increases greatly. This means that ID switches are avoided thanks to lifted edges.
Furthermore, the experiment shows that our designed features together with the lifted disjoint paths formulation \eqref{eq:lifted-disjoint-paths-problem} are well-suited for the MOT problem delivering nearly optimal assignments.

\begin{table}[hbt]

    \centering
    \tabcolsep=0.09cm
\begin{tabular}{lcccccc}
\toprule  &  $0.3\mathrm{s}$ & $0.5\mathrm{s}$ & $1\mathrm{s}$ & $1.5\mathrm{s}$ & $2\mathrm{s}$ & $\infty$ \\ \hline
MOTA (ours)$\uparrow$ & $52.6$ &  $52.7$ &  $52.8$ &  $52.8$ &  $\mathbf{52.8}$ & - \\
MOTA (optimal)$\uparrow$ & $53.0$&  $53.1$&  $53.3$&  $53.3$&  $\mathbf{53.4}$ & $\mathbf{53.4}$\\ \hline 
IDF1 (ours) $\uparrow$ & $55.7$ &  $57.8$ &  $61.8$ &  $63.8$ &  $\mathbf{64.3}$ & - \\
IDF1 (optimal)$\uparrow$ & $56.0$&  $58.6$&  $63.2$& $65.7$&  $66.8$ & $\mathbf{69.9}$\\ \hline
IDP (ours) $\uparrow$ & $79.8$ &  $82.9$ &  $88.5$ &  $91.4$ &  $\mathbf{92.1}$ & - \\
IDP (optimal)$\uparrow$ & $80.4$&  $84.2$&  $90.8$& $94.3$&  $ 95.9$ & $\mathbf{100.0}$  \\ \hline
IDR (ours) $\uparrow$ & $42.7$ &  $44.5$ &  $47.4$ &  $49.0$ &  $\mathbf{49.4}$ & - \\
IDR (optimal)$\uparrow$ & $42.9$&  $45.0$&  $48.5$& $50.4$&  $51.3$ & $\mathbf{53.4}$\\ \bottomrule
\end{tabular}   
\caption{Assignment quality of our solver without interpolation or extrapolation on the MOT17 train set with different maximal time gaps in seconds. Rows 1,3,5 and 7 show the results by our solver, rows 2,4,6 and 8 show the maximally achievable bounds with admissible assignment hypotheses up to the specified time gap. 
Bold numbers represent the best values per row.}
    \label{tab:tracking_performance_over_time}

\end{table}{}

\begin{table*}
\center

\small{

    \begin{tabular}{c l c c c c c c c c c}
     \toprule
     & Method  & MOTA$\uparrow$  & IDF1$\uparrow$  & MT$\uparrow$  & ML$\downarrow$  & FP$\downarrow$ & FN$\downarrow$ & IDS$\downarrow$ & Frag$\downarrow$ \\   
     \toprule
     \parbox[t]{3mm}{\multirow{5}{*}{\rotatebox[origin=c]{90}{MOT17}}} &
 Lif\_T (ours) &  $\mathbf{60.5}$ &$ \mathbf{65.6} $& $27.0 $ & $33.6$  & $14966$ & $\mathbf{206619}$ & $1189 $& $3476$  \\ 
 & Lif\_TsimInt (ours) &  $58.2$ &$ 65.2 $& $\mathbf{28.6} $ & $33.6$  & $16850$ & $217944$ & $\mathbf{1022}$& $\mathbf{2062}$  \\ 
& Tracktor17&  $53.5$ & $52.3$ & $19.5$ & $36.6 $ & $\mathbf{12201}$ & $248047$ & $2072$ & $4611$  \\ 
& JBNOT~ &  $52.6$ & $50.8$ & $19.7$ & $35.8$ & $31572$ & $232659$ & $3050$ & $3792$  \\ 
& FAMNet&  $52.0$ & $48.7$ & $19.1$ & $\mathbf{33.4}$  & $14138$ & $253616$ & $3072$ & $5318$  \\ 
& eTC17& $51.9$ & $58.1$ & $23.1$  & $35.5$  & $36164$ & $232783$ & $2288$ & $3071$  \\ 
& eHAF17&  $51.8$ & $54.7$ & $23.4$  & $37.9$  & $33212$ & $236772$ & $1834$ & $2739 $  \\

     \midrule
     
     \parbox[t]{3mm}{\multirow{5}{*}{\rotatebox[origin=c]{90}{MOT16}}} &
     Lif\_T (ours) &  $\mathbf{61.3} $& $\mathbf{64.7}$ & $\mathbf{27.0}$ & $34.0$ & $4844$ & $\mathbf{65401}$ & $389$ & $1034$  \\ 
     & Lif\_TsimInt (ours) &  $57.5$ &$ 64.1 $& $25.4 $ & $\mathbf{34.7}$  & $4249$ & $72868$ & $\mathbf{335}$& $604$  \\
& Tracktor16 &  $54.4$ & $52.5$ & $19.0$ & $36.9$ & $\mathbf{3280}$ & $79149$ & $682$ & $1480$  \\ 
& NOTA &  $49.8$ & $55.3$ & $17.9$ & $37.7$ & $7248$ & $83614$ & $614$ & $1372$  \\ 
& HCC &  $49.3$ & $50.7$ & $17.8$ & $39.9$ & $5333$ & $86795$ & $391$ & $\mathbf{535}$  \\ 
& eTC&  $49.2$ & $56.1$ & $17.3$ & $40.3$ & $8400$ & $83702$ & $606$ & $882$  \\ 
& KCF16  &  $48.8$ & $47.2$ & $15.8$ & $38.1$ & $5875$ & $86567$ & $906$ & $1116$  \\
     \midrule
     
     \parbox[t]{3mm}{\multirow{5}{*}{\rotatebox[origin=c]{90}{2D MOT15}}} &
    Lif\_T (ours) &  $\mathbf{52.5}$ & $\mathbf{60.0}$ & $\mathbf{33.8} $ & $\mathbf{25.8}$ & $6837$ & $\mathbf{21610}$ & $730$ & $1047 $ \\ 
    & Lif\_TsimInt (ours) &  $47.2$ & $57.6$ & $27.0 $ & $29.8$ & $7635$ & $24277$ & $\mathbf{554}$ & $\mathbf{803} $ \\ 
& Tracktor15 & $ 44.1 $& $46.7$ & $18.0$ & $26.2$ & $6477$ & $26577$ & $1318$ & $1790$  \\ 
& KCF  & $38.9$ & $44.5$ & $16.6$ & $31.5$ & $7321$ & $29501$ & $720$ & $1440$  \\ 
& AP\_HWDPL\_p &  $38.5$ & $47.1$ & $8.7$ & $37.4$  & $\mathbf{4005}$ & $33203$ & $586$ & $1263$  \\ 
& STRN &  $38.1$ & $46.6$ & $11.5$ & $33.4$  & $5451$ & $31571$ & $1033$ & $2665$  \\ 
& AMIR15 &  $37.6$ & $46.0$ & $15.8$ & $26.8$ & $7933$ & $29397$ & $1026$ & $2024$    \\
     \bottomrule
    \end{tabular}

\caption{We compare our tracker Lif\_T with the five best performing competing solvers w.r.t.\ MOTA from the MOT challenge. 
Tracktor~\cite{bergmann2019tracking}, 
JBNOT~\cite{henschel2019multiple},
FAMNet~\cite{chu2019famnet},
eTC~\cite{wang2019exploit},
eHAF~\cite{sheng2018heterogeneous},
NOTA~\cite{chen2019aggregate},
HCC~\cite{ma2018customized},
KCF~\cite{chu2019online},
AP\_HWDPL\_p~\cite{chen2017online},
STRN~\cite{xu2019spatial} and
AMIR15~\cite{sadeghian2017tracking}.
In addition, we compare the results to our tracker Lif\_TsimInt that uses only a simple interpolation method (linear interpolation) as post-processing in all sequences.
We outperform competing solvers on most metrics on all three MOT Challenge benchmarks, using Lif\_T and Lif\_TsimInt.  
Arrows indicate whether low or high metric values are better.}
\vspace{-0.2cm}
\label{tab:mot}
}
\end{table*}

\subsection{Benchmark Evaluations}
Finally, we compare our tracking performance on the MOT15, MOT16 and MOT17 benchmarks with all trackers listed on the MOTChallenge which have been peer-reviewed and correspond to published work. 
The three benchmark datasets consist of 11/7/7 training and test sequences for MOT15/16/17 respectively.
They are the standard benchmark datasets for MOT.
The results in Table~\ref{tab:mot} show the tracking performance of our tracker together with the best 5 performing trackers, accumulated over all sequences of the respective benchmarks. 
The evaluations show that we outperform all tracking systems by a large margin on all considered benchmarks. On MOT17, we improve the MOTA score from 53.5 to 60.5 and the IDF1 score from  52.3 to 65.6, which corresponds to an improvement of 13\% in terms of MOTA and almost 25\% in terms of the IDF1 score, indicating the effectiveness of the lifted edges. We observe similar improvements across all three benchmarks. These results reflect the near-optimal assignment performance observed on the MOT17 train set in Sect.~\ref{sec:exp_long_range_edges}. 
Finally, using only simple linear interpolation as post-processing (Lif\_TsimInt), our tracker achieves  $58.2$ MOTA and $65.2$ IDF1. Even then, our system clearly outperforms existing tracking systems. On average, the ILP solver needs $26.6$ min.\ per sequence. Detailed runtimes are available in Table 5 in Appendix.

\section{Conclusion}
\label{sec:conclusion}
We have shown that for the MOT challenge datasets we reach nearly optimal data association performance.
We conjecture that further improvements would have to come from better detectors, better inter- and extrapolation and more powerful solvers for our formulation to take into account even longer time-gaps.
Our polyhedral work offers the basis for writing such more powerful solvers.

\section{Acknowledgement}
\label{sec:acknowledgement}
This work was funded by the Deutsche Forschungsgemeinschaft (DFG, German Research Foundation) under Germany's Excellence Strategy within the Cluster of Excellence PhoenixD (EXC 2122). We thank Laura Leal-Taix{\'e} for initiating the collaboration. We thank all reviewers for their valuable comments.
 
\clearpage 
\newpage
\bibliography{literature}

\begin{thebibliography}{59}
\providecommand{\natexlab}[1]{#1}
\providecommand{\url}[1]{\texttt{#1}}
\expandafter\ifx\csname urlstyle\endcsname\relax
  \providecommand{\doi}[1]{doi: #1}\else
  \providecommand{\doi}{doi: \begingroup \urlstyle{rm}\Url}\fi

\bibitem[Ahuja et~al.(1988)Ahuja, Magnanti, and Orlin]{ahuja1988network}
Ahuja, R.~K., Magnanti, T.~L., and Orlin, J.~B.
\newblock \emph{Network flows}.
\newblock Cambridge, Mass.: Alfred P. Sloan School of Management,
  Massachusetts, 1988.

\bibitem[Babaee et~al.(2018)Babaee, Athar, and Rigoll]{babaee2018multiple}
Babaee, M., Athar, A., and Rigoll, G.
\newblock Multiple people tracking using hierarchical deep tracklet
  re-identification.
\newblock \emph{arXiv preprint arXiv:1811.04091}, 2018.

\bibitem[Beier et~al.(2017)Beier, Pape, Rahaman, Prange, Berg, Bock, Cardona,
  Knott, Plaza, Scheffer, et~al.]{beier2017multicut}
Beier, T., Pape, C., Rahaman, N., Prange, T., Berg, S., Bock, D.~D., Cardona,
  A., Knott, G.~W., Plaza, S.~M., Scheffer, L.~K., et~al.
\newblock Multicut brings automated neurite segmentation closer to human
  performance.
\newblock \emph{Nature Methods}, 14\penalty0 (2):\penalty0 101, 2017.

\bibitem[Berclaz et~al.(2011)Berclaz, Fleuret, Turetken, and
  Fua]{berclaz2011multiple}
Berclaz, J., Fleuret, F., Turetken, E., and Fua, P.
\newblock Multiple object tracking using k-shortest paths optimization.
\newblock \emph{IEEE Transactions on Pattern Analysis and Machine
  Intelligence}, 33\penalty0 (9):\penalty0 1806--1819, 2011.

\bibitem[Bergmann et~al.(2019)Bergmann, Meinhardt, and
  Leal-Taix\'{e}]{bergmann2019tracking}
Bergmann, P., Meinhardt, T., and Leal-Taix\'{e}, L.
\newblock Tracking without bells and whistles.
\newblock In \emph{IEEE International Conference on Computer Vision}, pp.\
  941--951, 2019.

\bibitem[Bernardin \& Stiefelhagen(2008)Bernardin and
  Stiefelhagen]{bernardin2008evaluating}
Bernardin, K. and Stiefelhagen, R.
\newblock Evaluating multiple object tracking performance: the clear mot
  metrics.
\newblock \emph{EURASIP Journal on Image and Video Processing}, 2008:\penalty0
  1--10, 2008.

\bibitem[Brendel et~al.(2011)Brendel, Amer, and
  Todorovic]{brendel2011multiobject}
Brendel, W., Amer, M., and Todorovic, S.
\newblock Multiobject tracking as maximum weight independent set.
\newblock In \emph{IEEE Conference on Computer Vision and Pattern Recognition},
  pp.\  1273--1280. IEEE, 2011.

\bibitem[Chari et~al.(2015)Chari, Lacoste-Julien, Laptev, and
  Sivic]{chari2015pairwise}
Chari, V., Lacoste-Julien, S., Laptev, I., and Sivic, J.
\newblock On pairwise costs for network flow multi-object tracking.
\newblock In \emph{IEEE Conference on Computer Vision and Pattern Recognition},
  pp.\  5537--5545, 2015.

\bibitem[Chen et~al.(2017)Chen, Ai, Shang, Zhuang, and Bai]{chen2017online}
Chen, L., Ai, H., Shang, C., Zhuang, Z., and Bai, B.
\newblock Online multi-object tracking with convolutional neural networks.
\newblock In \emph{IEEE International Conference on Image Processing}, pp.\
  645--649. IEEE, 2017.

\bibitem[Chen et~al.(2019)Chen, Ai, Chen, and Zhuang]{chen2019aggregate}
Chen, L., Ai, H., Chen, R., and Zhuang, Z.
\newblock Aggregate tracklet appearance features for multi-object tracking.
\newblock \emph{IEEE Signal Processing Letters}, 26\penalty0 (11):\penalty0
  1613--1617, 2019.

\bibitem[Chu \& Ling(2019)Chu and Ling]{chu2019famnet}
Chu, P. and Ling, H.
\newblock Famnet: Joint learning of feature, affinity and multi-dimensional
  assignment for online multiple object tracking.
\newblock In \emph{IEEE International Conference on Computer Vision}, pp.\
  6172--6181, 2019.

\bibitem[Chu et~al.(2019)Chu, Fan, Tan, and Ling]{chu2019online}
Chu, P., Fan, H., Tan, C.~C., and Ling, H.
\newblock Online multi-object tracking with instance-aware tracker and dynamic
  model refreshment.
\newblock In \emph{IEEE Winter Conference on Applications of Computer Vision},
  pp.\  161--170. IEEE, 2019.

\bibitem[Cook(1971)]{cook1971complexity}
Cook, S.~A.
\newblock The complexity of theorem-proving procedures.
\newblock In \emph{Proceedings of the third annual ACM symposium on Theory of
  computing}, pp.\  151--158. ACM, 1971.

\bibitem[Dehghan et~al.(2015)Dehghan, Modiri~Assari, and
  Shah]{dehghan2015gmmcp}
Dehghan, A., Modiri~Assari, S., and Shah, M.
\newblock {GMMCP} tracker: Globally optimal generalized maximum multi clique
  problem for multiple object tracking.
\newblock In \emph{IEEE Conference on Computer Vision and Pattern Recognition},
  pp.\  4091--4099, 2015.

\bibitem[Eilam-Tzoreff(1998)]{eilam1998disjoint}
Eilam-Tzoreff, T.
\newblock The disjoint shortest paths problem.
\newblock \emph{Discrete Applied Mathematics}, 85\penalty0 (2):\penalty0
  113--138, 1998.

\bibitem[Even et~al.(1976)Even, Itai, and Shamir]{EvenMulti}
Even, S., Itai, A., and Shamir, A.
\newblock On the complexity of timetable and multicommodity flow problems.
\newblock \emph{SIAM J. Comput.}, 5:\penalty0 691--703, 12 1976.
\newblock \doi{10.1137/0205048}.

\bibitem[Gurobi~Optimization(2019)]{gurobi}
Gurobi~Optimization, L.
\newblock Gurobi optimizer reference manual, 2019.
\newblock URL \url{http://www.gurobi.com}.

\bibitem[Henschel et~al.(2014)Henschel, Leal-Taix{\'e}, and
  Rosenhahn]{henschel2014efficient}
Henschel, R., Leal-Taix{\'e}, L., and Rosenhahn, B.
\newblock Efficient multiple people tracking using minimum cost arborescences.
\newblock In \emph{German Conference on Pattern Recognition}, pp.\  265--276.
  Springer, 2014.

\bibitem[Henschel et~al.(2016)Henschel, Leal-Taix{\'e}, Rosenhahn, and
  Schindler]{henschel2016tracking}
Henschel, R., Leal-Taix{\'e}, L., Rosenhahn, B., and Schindler, K.
\newblock Tracking with multi-level features.
\newblock \emph{arXiv preprint arXiv:1607.07304}, 2016.

\bibitem[Henschel et~al.(2018)Henschel, Leal-Taix\'{e}, Cremers, and
  Rosenhahn]{Henschel_2018_CVPR_Workshops}
Henschel, R., Leal-Taix\'{e}, L., Cremers, D., and Rosenhahn, B.
\newblock Fusion of head and full-body detectors for multi-object tracking.
\newblock In \emph{IEEE Conference on Computer Vision and Pattern Recognition
  Workshops}, June 2018.

\bibitem[Henschel et~al.(2019{\natexlab{a}})Henschel, von Marcard, and
  Rosenhahn]{henschel2019simultaneous}
Henschel, R., von Marcard, T., and Rosenhahn, B.
\newblock Simultaneous identification and tracking of multiple people using
  video and imus.
\newblock In \emph{Proceedings of the IEEE Conference on Computer Vision and
  Pattern Recognition Workshops}, pp.\  0--0, 2019{\natexlab{a}}.

\bibitem[Henschel et~al.(2019{\natexlab{b}})Henschel, Zou, and
  Rosenhahn]{henschel2019multiple}
Henschel, R., Zou, Y., and Rosenhahn, B.
\newblock Multiple people tracking using body and joint detections.
\newblock In \emph{IEEE Conference on Computer Vision and Pattern Recognition
  Workshops}, pp.\  0--0, 2019{\natexlab{b}}.

\bibitem[Hofmann et~al.(2013)Hofmann, Wolf, and Rigoll]{hofmann2013hypergraphs}
Hofmann, M., Wolf, D., and Rigoll, G.
\newblock Hypergraphs for joint multi-view reconstruction and multi-object
  tracking.
\newblock In \emph{IEEE Conference on Computer Vision and Pattern Recognition},
  pp.\  3650--3657, 2013.

\bibitem[Hor\v{n}\'akov\'a et~al.(2017)Hor\v{n}\'akov\'a, Lange, and
  Andres]{hornakova2017analysis}
Hor\v{n}\'akov\'a, A., Lange, J.-H., and Andres, B.
\newblock Analysis and optimization of graph decompositions by lifted
  multicuts.
\newblock In \emph{International Conference on Machine Learning}, 2017.

\bibitem[Hu et~al.(2019)Hu, Shi, Zhou, Xing, Ling, and Maybank]{hu2019dual}
Hu, W., Shi, X., Zhou, Z., Xing, J., Ling, H., and Maybank, S.
\newblock Dual {L}1-normalized context aware tensor power iteration and its
  applications to multi-object tracking and multi-graph matching.
\newblock \emph{International Journal of Computer Vision}, Oct 2019.
\newblock ISSN 1573-1405.
\newblock \doi{10.1007/s11263-019-01231-y}.
\newblock URL \url{https://doi.org/10.1007/s11263-019-01231-y}.

\bibitem[Keuper et~al.(2015)Keuper, Levinkov, Bonneel, Lavou\'{e}, Brox, and
  Andres]{keuper2015lifted}
Keuper, M., Levinkov, E., Bonneel, N., Lavou\'{e}, G., Brox, T., and Andres, B.
\newblock Efficient decomposition of image and mesh graphs by lifted multicuts.
\newblock In \emph{IEEE International Conference on Computer Vision}, 2015.
\newblock \doi{10.1109/ICCV.2015.204}.

\bibitem[Keuper et~al.(2016)Keuper, Tang, Zhongjie, Andres, Brox, and
  Schiele]{keuper2016multi}
Keuper, M., Tang, S., Zhongjie, Y., Andres, B., Brox, T., and Schiele, B.
\newblock A multi-cut formulation for joint segmentation and tracking of
  multiple objects.
\newblock \emph{arXiv preprint arXiv:1607.06317}, 2016.

\bibitem[Keuper et~al.(2018)Keuper, Tang, Andres, Brox, and
  Schiele]{keuper2018motion}
Keuper, M., Tang, S., Andres, B., Brox, T., and Schiele, B.
\newblock Motion segmentation \& multiple object tracking by correlation
  co-clustering.
\newblock \emph{IEEE Transactions on Pattern Analysis and Machine
  Intelligence}, 42\penalty0 (1):\penalty0 140--153, 2018.

\bibitem[Kirillov et~al.(2017)Kirillov, Levinkov, Andres, Savchynskyy, and
  Rother]{kirillov2017instancecut}
Kirillov, A., Levinkov, E., Andres, B., Savchynskyy, B., and Rother, C.
\newblock Instancecut: from edges to instances with multicut.
\newblock In \emph{IEEE Conference on Computer Vision and Pattern Recognition},
  pp.\  5008--5017, 2017.

\bibitem[Kov{\'a}cs(2015)]{kovacs2015minimum}
Kov{\'a}cs, P.
\newblock Minimum-cost flow algorithms: an experimental evaluation.
\newblock \emph{Optimization Methods and Software}, 30\penalty0 (1):\penalty0
  94--127, 2015.

\bibitem[Kumar et~al.(2014)Kumar, Charpiat, and Thonnat]{kumar2014multiple}
Kumar, R., Charpiat, G., and Thonnat, M.
\newblock Multiple object tracking by efficient graph partitioning.
\newblock In \emph{Asian Conference on Computer Vision}, pp.\  445--460.
  Springer, 2014.

\bibitem[Leal-Taix\'{e} et~al.(2012)Leal-Taix\'{e}, Pons-Moll, and
  Rosenhahn]{leal2012branch}
Leal-Taix\'{e}, L., Pons-Moll, G., and Rosenhahn, B.
\newblock Branch-and-price global optimization for multi-view multi-target
  tracking.
\newblock In \emph{IEEE Conference on Computer Vision and Pattern Recognition},
  pp.\  1987--1994. IEEE, 2012.

\bibitem[Leal-Taix\'{e} et~al.(2015)Leal-Taix\'{e}, Milan, Reid, Roth, and
  Schindler]{MOTChallenge2015}
Leal-Taix\'{e}, L., Milan, A., Reid, I., Roth, S., and Schindler, K.
\newblock {MOTC}hallenge 2015: {T}owards a benchmark for multi-target tracking.
\newblock \emph{arXiv:1504.01942 [cs]}, April 2015.
\newblock URL \url{http://arxiv.org/abs/1504.01942}.
\newblock arXiv: 1504.01942.

\bibitem[Levinkov et~al.(2017)Levinkov, Uhrig, Tang, Omran, Insafutdinov,
  Kirillov, Rother, Brox, Schiele, and Andres]{levinkov2017joint}
Levinkov, E., Uhrig, J., Tang, S., Omran, M., Insafutdinov, E., Kirillov, A.,
  Rother, C., Brox, T., Schiele, B., and Andres, B.
\newblock Joint graph decomposition \& node labeling: Problem, algorithms,
  applications.
\newblock In \emph{IEEE Conference on Computer Vision and Pattern Recognition},
  pp.\  6012--6020, 2017.

\bibitem[Li et~al.(2009)Li, Huang, and Nevatia]{li2009learning}
Li, Y., Huang, C., and Nevatia, R.
\newblock Learning to associate: Hybridboosted multi-target tracker for crowded
  scene.
\newblock In \emph{IEEE Conference on Computer Vision and Pattern Recognition},
  pp.\  2953--2960. IEEE, 2009.

\bibitem[Ma et~al.(2018)Ma, Tang, Black, and Van~Gool]{ma2018customized}
Ma, L., Tang, S., Black, M.~J., and Van~Gool, L.
\newblock Customized multi-person tracker.
\newblock In \emph{Asian Conference on Computer Vision}, pp.\  612--628.
  Springer, 2018.

\bibitem[Milan et~al.(2016)Milan, Leal-Taix\'{e}, Reid, Roth, and
  Schindler]{MOT16}
Milan, A., Leal-Taix\'{e}, L., Reid, I., Roth, S., and Schindler, K.
\newblock {MOT}16: {A} benchmark for multi-object tracking.
\newblock \emph{arXiv:1603.00831 [cs]}, March 2016.
\newblock URL \url{http://arxiv.org/abs/1603.00831}.
\newblock arXiv: 1603.00831.

\bibitem[Nowozin \& Lampert(2010)Nowozin and Lampert]{nowozin2010global}
Nowozin, S. and Lampert, C.~H.
\newblock Global interactions in random field models: A potential function
  ensuring connectedness.
\newblock \emph{SIAM Journal on Imaging Sciences}, 3\penalty0 (4):\penalty0
  1048--1074, 2010.
\newblock \doi{10.1137/090752614}.

\bibitem[Rempfler et~al.(2017)Rempfler, Lange, Jug, Blasse, Myers, Menze, and
  Andres]{rempfler2017efficient}
Rempfler, M., Lange, J.-H., Jug, F., Blasse, C., Myers, E.~W., Menze, B.~H.,
  and Andres, B.
\newblock Efficient algorithms for moral lineage tracing.
\newblock In \emph{IEEE International Conference on Computer Vision}, pp.\
  4695--4704, 2017.

\bibitem[Ristani \& Tomasi(2014)Ristani and Tomasi]{ristani2014tracking}
Ristani, E. and Tomasi, C.
\newblock Tracking multiple people online and in real time.
\newblock In \emph{Asian Conference on Computer Vision}, pp.\  444--459.
  Springer, 2014.

\bibitem[Ristani \& Tomasi(2018)Ristani and Tomasi]{ristani2018features}
Ristani, E. and Tomasi, C.
\newblock Features for multi-target multi-camera tracking and
  re-identification.
\newblock In \emph{IEEE Conference on Computer Vision and Pattern Recognition},
  pp.\  6036--6046, 2018.

\bibitem[Ristani et~al.(2016)Ristani, Solera, Zou, Cucchiara, and
  Tomasi]{ristani2016MTMC}
Ristani, E., Solera, F., Zou, R.~S., Cucchiara, R., and Tomasi, C.
\newblock Performance measures and a data set for multi-target, multi-camera
  tracking.
\newblock In \emph{European Conference on Computer Vision Workshop on
  Benchmarking Multi-Target Tracking}, 2016.

\bibitem[Sadeghian et~al.(2017)Sadeghian, Alahi, and
  Savarese]{sadeghian2017tracking}
Sadeghian, A., Alahi, A., and Savarese, S.
\newblock Tracking the untrackable: Learning to track multiple cues with
  long-term dependencies.
\newblock In \emph{IEEE International Conference on Computer Vision}, pp.\
  300--311, 2017.

\bibitem[Sheng et~al.(2018)Sheng, Zhang, Chen, Xiong, and
  Zhang]{sheng2018heterogeneous}
Sheng, H., Zhang, Y., Chen, J., Xiong, Z., and Zhang, J.
\newblock Heterogeneous association graph fusion for target association in
  multiple object tracking.
\newblock \emph{IEEE Transactions on Circuits and Systems for Video
  Technology}, 29\penalty0 (11):\penalty0 3269--3280, 2018.

\bibitem[Suurballe(1974)]{suurballe1974disjoint}
Suurballe, J.
\newblock Disjoint paths in a network.
\newblock \emph{Networks}, 4\penalty0 (2):\penalty0 125--145, 1974.

\bibitem[Tang et~al.(2015)Tang, Andres, Andriluka, and
  Schiele]{tang2015subgraph}
Tang, S., Andres, B., Andriluka, M., and Schiele, B.
\newblock Subgraph decomposition for multi-target tracking.
\newblock In \emph{IEEE Conference on Computer Vision and Pattern Recognition},
  pp.\  5033--5041, 2015.

\bibitem[Tang et~al.(2016)Tang, Andres, Andriluka, and Schiele]{tang2016multi}
Tang, S., Andres, B., Andriluka, M., and Schiele, B.
\newblock Multi-person tracking by multicut and deep matching.
\newblock In \emph{European Conference on Computer Vision}, pp.\  100--111.
  Springer, 2016.

\bibitem[Tang et~al.(2017)Tang, Andriluka, Andres, and
  Schiele]{tang2017multiple}
Tang, S., Andriluka, M., Andres, B., and Schiele, B.
\newblock Multiple people tracking by lifted multicut and person
  re-identification.
\newblock In \emph{IEEE Conference on Computer Vision and Pattern Recognition},
  2017.

\bibitem[Tholey(2012)]{tholey2012linear}
Tholey, T.
\newblock Linear time algorithms for two disjoint paths problems on directed
  acyclic graphs.
\newblock \emph{Theoretical Computer Science}, 465:\penalty0 35--48, 2012.

\bibitem[von Marcard et~al.(2018)von Marcard, Henschel, Black, Rosenhahn, and
  Pons-Moll]{von2018recovering}
von Marcard, T., Henschel, R., Black, M.~J., Rosenhahn, B., and Pons-Moll, G.
\newblock Recovering accurate 3d human pose in the wild using imus and a moving
  camera.
\newblock In \emph{Proceedings of the European Conference on Computer Vision
  (ECCV)}, pp.\  601--617, 2018.

\bibitem[Wang et~al.(2019{\natexlab{a}})Wang, Wang, Wang, Wu, and
  Yu]{wang2019mussp}
Wang, C., Wang, Y., Wang, Y., Wu, C.-T., and Yu, G.
\newblock mussp: Efficient min-cost flow algorithm for multi-object tracking.
\newblock In \emph{Advances in Neural Information Processing Systems}, pp.\
  423--432, 2019{\natexlab{a}}.

\bibitem[Wang et~al.(2019{\natexlab{b}})Wang, Wang, Zhang, Gu, and
  Hwang]{wang2019exploit}
Wang, G., Wang, Y., Zhang, H., Gu, R., and Hwang, J.-N.
\newblock Exploit the connectivity: Multi-object tracking with trackletnet.
\newblock In \emph{ACM International Conference on Multimedia}, pp.\  482--490,
  2019{\natexlab{b}}.

\bibitem[Wei et~al.(2018)Wei, Zhang, Gao, and Tian]{wei2018person}
Wei, L., Zhang, S., Gao, W., and Tian, Q.
\newblock Person transfer gan to bridge domain gap for person
  re-identification.
\newblock In \emph{IEEE Conference on Computer Vision and Pattern Recognition},
  pp.\  79--88, 2018.

\bibitem[Weinzaepfel et~al.(2013)Weinzaepfel, Revaud, Harchaoui, and
  Schmid]{weinzaepfel:hal-00873592}
Weinzaepfel, P., Revaud, J., Harchaoui, Z., and Schmid, C.
\newblock Deepflow: Large displacement optical flow with deep matching.
\newblock In \emph{IEEE Intenational Conference on Computer Vision}, Sydney,
  Australia, December 2013.
\newblock URL \url{http://hal.inria.fr/hal-00873592}.

\bibitem[Xu et~al.(2019)Xu, Cao, Zhang, and Hu]{xu2019spatial}
Xu, J., Cao, Y., Zhang, Z., and Hu, H.
\newblock Spatial-temporal relation networks for multi-object tracking.
\newblock In \emph{IEEE International Conference on Computer Vision}, pp.\
  3988--3998, 2019.

\bibitem[Zamir et~al.(2012)Zamir, Dehghan, and Shah]{zamir2012gmcp}
Zamir, A.~R., Dehghan, A., and Shah, M.
\newblock {GMCP}-tracker: Global multi-object tracking using generalized
  minimum clique graphs.
\newblock In \emph{European Conference on Computer Vision}, pp.\  343--356.
  Springer, 2012.

\bibitem[Zhang et~al.(2008)Zhang, Li, and Nevatia]{zhang2008global}
Zhang, L., Li, Y., and Nevatia, R.
\newblock Global data association for multi-object tracking using network
  flows.
\newblock In \emph{IEEE Conference on Computer Vision and Pattern Recognition},
  pp.\  1--8. IEEE, 2008.

\bibitem[Zheng et~al.(2015)Zheng, Shen, Tian, Wang, Wang, and
  Tian]{zheng2015scalable}
Zheng, L., Shen, L., Tian, L., Wang, S., Wang, J., and Tian, Q.
\newblock Scalable person re-identification: A benchmark.
\newblock In \emph{IEEE International Conference on Computer Vision}, pp.\
  1116--1124, 2015.

\bibitem[Zheng et~al.(2019)Zheng, Yang, Yu, Zheng, Yang, and
  Kautz]{zheng2019joint}
Zheng, Z., Yang, X., Yu, Z., Zheng, L., Yang, Y., and Kautz, J.
\newblock Joint discriminative and generative learning for person
  re-identification.
\newblock In \emph{IEEE Conference on Computer Vision and Pattern Recognition},
  2019.

\end{thebibliography}
\bibliographystyle{icml2020}

\setcounter{@affiliationcounter}{2}

\clearpage
\newpage
\twocolumn[
\icmltitle{Lifted Disjoint Paths with Application in Multiple Object Tracking \\ \normalfont Appendix}



\icmlsetsymbol{equal}{*}

\begin{icmlauthorlist}
\icmlauthor{Andrea Hornakova}{equal,mpi}
\icmlauthor{Roberto Henschel}{equal,tnt}
\icmlauthor{Bodo Rosenhahn}{tnt}
\icmlauthor{Paul Swoboda}{mpi}
\end{icmlauthorlist}

\icmlaffiliation{mpi}{Computer Vision and Machine Learning, Max Planck Institute for Informatics, Saarbr\"ucken, Saarland, Germany}


\icmlkeywords{combinatorial optimization, disjoint paths, network flow, polyhedral relaxation, multiple object tracking}

\vskip 0.3in
]
\printAffiliationsAndNotice{\icmlEqualContribution} %

\begin{abstract}
This appendix supplements our work by presenting missing proofs regarding the solver and details about our tracker. 

Sections \ref{sec:appendix-constraints} up to Section \ref{sec:appendix-complexity} provide proofs used in Sections \ref{sec:constraints} and \ref{sec:complexity}. 

Section \ref{sec:optimal_assignment_definition} provides further information how  the optimal assignments used in Section \ref{sec:exp_long_range_edges} were obtained. The impact of the employed post-processing used in our tracker is analyzed in Section \ref{sec:ablation_postprocessing}. Details about the used fusion network are given in Section \ref{sec:fusion_network_details}. Finally, evaluation metrics for all tracked sequences are provided in Section \ref{sec:results_all_sequences}. 
\end{abstract}

\section{Appendix}
\subsection{Proofs for Section~\ref{sec:constraints}}
\label{sec:appendix-constraints}

\begin{restatable}{prop}{multicutpathineqweaker}
\label{prop:multicut-path-inequalities weaker}
Path inequalities \eqref{eq:path-inequalities} define a strictly tighter relaxation of the lifted disjoint path problem than the lifted multicut path inequalities
\begin{align}
\nonumber \forall vw\in E'\ \forall P\in &vw\mhyphen\text{paths}(G):\\
\label{eq:multicut-path}y'_{vw}&\geq \sum_{ij\in P_E}(y_{ij}-1)+1\,.
\end{align}
\end{restatable}
\begin{proof}
Let us define the following sets:
\begin{align*}
S_B&=\{(y,y')\in [0,1]^{E}\times [0,1]^{E'}|(y,y') \text{ satisfy } \eqref{eq:path-inequalities}\}\,,\\
S_M&=\{(y,y')\in [0,1]^{E}\times [0,1]^{E'}|(y,y') \text{ satisfy } \eqref{eq:multicut-path} \}  \,. 
\end{align*}
\begin{itemize}
    \item Let us prove that $S_B\subset S_M$
    
    Let us rewrite the right hand side of \eqref{eq:path-inequalities} for a path $P\in vw$-paths$(G)$:
 \begin{align}
\nonumber y'_{vw}&\geq\sum_{vj:j\in P_V}y_{vj}- \sum_{i\in P_V\setminus \{v,w\}}\sum_{k\notin P_V}y_{ik} = \\ 
\nonumber&= \sum_{vj:j\in P_V}y_{vj}- \sum_{i\in P_V\setminus \{v,w\}}(x_i-\sum_{j\in P_V}y_{ij})=\\
\nonumber&=\sum_{i\in P_V\setminus w}\sum_{j\in P_V}y_{ij}-\sum_{i\in P_V\setminus \{v,w\}}x_i\geq\\
\nonumber&\geq \sum_{ij\in P_E}y_{ij}-\sum_{i\in P_V\setminus \{v,w\}}1=\\
&= \sum_{ij\in P_E}(y_{ij}-1)+1\,.
 \end{align}
\item Let us prove that  $S_B\subsetneq S_M$

We prove that there exists $(y,y')\in [0,1]^{E}\times [0,1]^{E'}$ such that $(y,y')$ satisfies \eqref{eq:multicut-path} and does not satisfy \eqref{eq:path-inequalities}. 
An example is given in Figure \ref{fig:multicut-path}.
There are four possible paths from $v$ to $w$.
If we use Constraints \eqref{eq:multicut-path}, the highest lower bound on $y'_{vw}$ is given by path $P=(vv_2,v_2v_4,v_4w)$ and it is as follows:
\begin{align*}
y'_{vw}\geq (0.5-1)+(0.5-1)+(1-1)+1=0\ .
\end{align*}
Let us apply Constraint \eqref{eq:path-inequalities} using path $P=(vv_1,v_1v_2,v_2v_3,v_3v_4,v_4w)$. We obtain the following threshold on $y'_{vw}$
\begin{align*}
y'_{vw}\geq 0.5+0.5-0-0=1\,.
\end{align*}
\end{itemize}
\end{proof}

\begin{figure}
\begin{center}
\begin{tikzpicture}

\node[style=vertex] (1) at (1, 0) {$v$};
\node[style=vertex] (2) at (2, 0) {$v_1$};
\node[style=vertex] (3) at (3, 0) {$v_2$};
\node[style=vertex] (4) at (4, 0) {$v_3$};
\node[style=vertex] (5) at (5, 0) {$v_4$};
\node[style=vertex] (6) at (6, 0) {$w$};

\node[style=vertex] (7) at (3, -0.7) {$v_5$};
\node[style=vertex] (8) at (6, -0.7) {$v_6$};

\draw[base-edge] (1) edge node[midway,above]{\footnotesize{0.5}} (2);
\draw[base-edge] (2) edge node[midway,above]{\footnotesize{0.5}}(3);
\draw[base-edge] (3) edge node[midway,above]{\footnotesize{0.5}}(4);
\draw[base-edge] (4) edge node[midway,above]{\footnotesize{0.5}}(5);
\draw[base-edge] (5) edge node[midway,above]{\footnotesize{1}}(6);

\draw[base-edge] (2) edge node[pos=0.3,below]{\footnotesize{0}}(7);
\draw[base-edge] (5) edge node[pos=0.3,below]{\footnotesize{0}}(8);

\draw[base-edge] (1) edge [bend left=50] node[midway,above]{\footnotesize{0.5}} (3);
\draw[base-edge] (3) edge [bend left=50] node[midway,above]{\footnotesize{0.5}} (5);

\draw[lifted-edge] (1) edge [bend left=60] node[midway,above]{\footnotesize{?}} (6);

\end{tikzpicture}
\caption{Failure case for lifted multicut path inequality~\eqref{eq:multicut-path}. The path inequality~\eqref{eq:path-inequalities} gives the correct lower bound for lifted edge $y'_{vw}$ in this case. Example for Proposition \ref{prop:multicut-path-inequalities weaker}.}
\label{fig:multicut-path}
\end{center}
\end{figure}
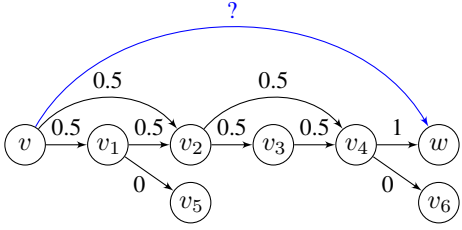

\liftedstrictlybetter*

\begin{proof}
Let us define the following sets
\begin{align*}
S_B&=\{(y,y')\in [0,1]^{E}\times [0,1]^{E'}|(y,y') \text{ satisfy } \eqref{eq:path-inequalities}\}\,,\\
S_L&=\{(y,y')\in [0,1]^{E}\times [0,1]^{E'}|(y,y') \text{ satisfy } \eqref{eq:lifted-path-inequalities}\}\,.    
\end{align*}
\begin{itemize}
\item Let us prove that $S_{L}\subset S_B$:\\
Note that every path $P\in vw\mhyphen\text{paths}(G)$ belongs to the set of $vw\mhyphen\text{paths}(G\cup G')$ too. It just holds that $P_{E'}=\emptyset$. 
Let us rewrite the right hands side of the inequality from~\eqref{eq:lifted-path-inequalities} for such  $P\in vw\mhyphen\text{path}(G\cup G')$ where $P_{E'}=\emptyset$. 
\begin{align*}
\nonumber y'_{vw}\geq&\sum_{vj:j\in P_V}y_{vj} - \sum_{i\in P_V\setminus \{v,w\}}\sum_{k\notin P_V} y_{ik}\\
& + \sum_{ij\in P_{E'}} y'_{ij}-\sum_{ij\in P_{E'}\cap E} y_{ij}=\\
=&\nonumber \sum_{vj: j\in P_V}y_{vj} - \sum_{i\in P_V\setminus \{v,w\}}\sum_{k\notin P_V} y_{ik}\,.
\end{align*}
Which is exactly the right hand side of~\eqref{eq:path-inequalities}. Therefore, any pair of real vectors $(y,y')\in [0,1]^{E}\times [0,1]^{E'}$ that satisfies~\eqref{eq:lifted-path-inequalities} must satisfy~\eqref{eq:path-inequalities} as well.

\item Let us prove that $S_L\subsetneq S_B$:\\
We prove that there exists $(y,y')\in [0,1]^{E}\times [0,1]^{E'}$ such that $(y,y')$ satisfies~\eqref{eq:path-inequalities} and does not satisfy~\eqref{eq:lifted-path-inequalities}. See the graph in Figure~\ref{fig:lifted-path-proof}. There are four possible paths from $v$ to $w$ in $G$. If we use Constraints~\eqref{eq:path-inequalities}, all the paths give us the same lower bound on $y'_{vw}$
\begin{align*}
y'_{vw}\geq 1-0.5-0.5=0\ .
\end{align*}
\end{itemize}
If we use Constraints \eqref{eq:lifted-path-inequalities} with path $P=(vv_1,v_1v_4,v_4w)$ where $P_{E'}=\{v_1v_4,v_4w\}$, we obtain
\begin{align*}
y'_{vw}\geq 1-0.5-0.5-0.5-0.5+1+1=1\ .
\end{align*}
\end{proof}

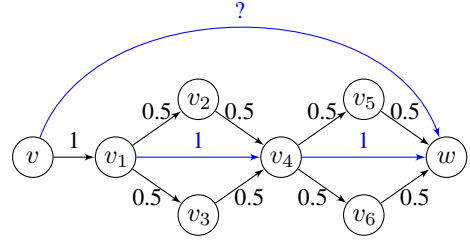
\begin{figure}
\begin{center}
\begin{tikzpicture}

\node[style=vertex] (2) at (1, 0) {$v$};
\node[style=vertex] (3) at (2, 0) {$v_1$};
\node[style=vertex] (4) at (3, 0.7) {$v_2$};
\node[style=vertex] (5) at (3, -0.7) {$v_3$};
\node[style=vertex] (6) at (4, 0) {$v_4$};
\node[style=vertex] (7) at (5, 0.7) {$v_5$};
\node[style=vertex] (8) at (5, -0.7) {$v_6$};
\node[style=vertex] (9) at (6, 0) {$w$};
\draw[base-edge] (2) -- (3) node[midway,above]{\footnotesize{1}} ;
\draw[base-edge] (3) -- (4)node[midway,above]{\footnotesize{0.5}} ;
\draw[base-edge] (3) -- (5)node[pos=0.3,below]{\footnotesize{0.5}} ;
\draw[base-edge] (4) -- (6)node[midway,above]{\footnotesize{0.5}} ;
\draw[base-edge] (5) -- (6)node[pos=0.7,below]{\footnotesize{0.5}} ;
\draw[lifted-edge] (3) -- (6)node[midway,above]{\footnotesize{1}} ;
\draw[base-edge] (6) -- (7)node[midway,above]{\footnotesize{0.5}} ;
\draw[base-edge] (6) -- (8)node[pos=0.3,below]{\footnotesize{0.5}} ;
\draw[base-edge] (7) -- (9)node[midway,above]{\footnotesize{0.5}} ;
\draw[base-edge] (8) -- (9)node[pos=0.7,below]{\footnotesize{0.5}} ;
\draw[lifted-edge] (6) -- (9)node[midway,above]{\footnotesize{1}} ;
\draw[lifted-edge] (2) edge [bend left=70] node[midway,above]{\footnotesize{?}} (9);

\end{tikzpicture}
\caption{Exemplary case where the path inequalities~\eqref{eq:path-inequalities} give a trivial lower bound on lifted edge $y'_{vw}$.
The lifted path inequality~\eqref{eq:lifted-path-inequalities} gives the correct lower bound. Example for Proposition~\ref{prop:lifted-path-strictly-better}.
}
\label{fig:lifted-path-proof}
\end{center}
\end{figure}

\liftedpathinducedcutbetter*

\begin{proof}
Let us define the following sets
\begin{align*}
S_B&=\{(y,y')\in [0,1]^{E}\times [0,1]^{E'}|(y,y') \text{ satisfy } \eqref{eq:path-induced-cut-inequality}\}\,,\\
S_{L1}&=\{(y,y')\in [0,1]^{E}\times [0,1]^{E'}|(y,y') \text{ satisfy } \eqref{eq:lifted-path-induced-cut-inequality}\}\,,\\
S_{L2}&=\{(y,y')\in [0,1]^{E}\times [0,1]^{E'}|(y,y') \text{ satisfy } \eqref{eq:lifted-path-induced-cut-inequality2}\}\,.
\end{align*}

\begin{itemize}
\item First, we prove $S_{L1}\subset S_B$:\\
We use the same argument as in the proof of Proposition~\ref{prop:lifted-path-strictly-better}. Every path $P\in vw\mhyphen\text{paths}(G)$ belongs to the set of $vw\mhyphen\text{paths}(G\cup G')$ and it holds that $P_{E'}=\emptyset$. 
Let us rewrite the right hands side of the inequality from~\eqref{eq:lifted-path-induced-cut-inequality} for such  $P\in vw\mhyphen\text{path}(G\cup G')$ where $P_{E'}=\emptyset$. 
\begin{align*}
y'_{vw}\leq&\sum_{i\in P_V}\sum_{\substack{k\notin P_V\\ kw \in \mathcal{R}}}y_{ik}-\sum_{ij\in P_{E'}}y'_{ij}
 +\sum_{ij\in P_{E'}\cap E}y_{ij}= \\
=&\sum_{i\in P_V}\sum_{\substack{k\notin P_V\\ kw \in \mathcal{R}}}y_{ik}\,.
\end{align*}
Which is exactly the right hand side of \eqref{eq:path-induced-cut-inequality}. Therefore, any pair of real vectors $(y,y')\in [0,1]^{E}\times [0,1]^{E'}$ that satisfies \eqref{eq:lifted-path-induced-cut-inequality} must satisfy \eqref{eq:path-induced-cut-inequality}.\\

\item Let us prove $S_{L1}\subsetneq S_B$:\\
We prove that there exists $(y,y')\in [0,1]^{E}\times [0,1]^{E'}$ such that $(y,y')$ satisfies \eqref{eq:path-induced-cut-inequality} and does not satisfy \eqref{eq:lifted-path-induced-cut-inequality}.

See the example in Figure \ref{fig:lifted-cut-proof}. There are four possible paths in $G$ from $v$ to either $u_1$ or $u_2$. They are $P_1=(vv_3,v_3u_1)$, $P_2=(vv_2,v_2u_1)$, $P_3=(vv_3,v_3u_2)$, $P_4=(vv_2,v_2u_2)$. Using~\eqref{eq:lifted-path-induced-cut-inequality}, all of them give us the same threshold on $y'_{vw}$:
\begin{align*}
y'_{vw}\leq 0.5+0.5+0=1\,.      
\end{align*}
If we use Constraint \eqref{eq:lifted-path-induced-cut-inequality} with path $P=(vu_1)$, we obtain the following threshold:
\begin{align*}
y'_{vw}\leq 0.5+0.5+0-1=0\,.      
\end{align*}

\item Let us prove that $S_{L1}\cap S_{L2}\subsetneq S_{L1}$\\
It holds trivially that $S_{L1}\cap S_{L2}\subset S_{L1}$. Let us prove that there exists  $(y,y')\in [0,1]^{E}\times [0,1]^{E'}$ such that $(y,y')\in S_{L1}$ and $(y,y')\notin S_{L1}\cap S_{L2}$.

See the example graph in Figure~\ref{fig:lifted-cut-proof2}. Similarly as in Figure~\ref{fig:lifted-cut-proof}, there are  four possible paths from $v$ to either $u_1$ or $u_2$ in $G$. There are no active lifted edges that would enable us to obtain a better upper bound on $y'_{vw}$ using~\eqref{eq:lifted-path-induced-cut-inequality} than the following:
\begin{align*}
y'_{vw}\leq 1  \,.      
\end{align*}
However, if we use Constraints \eqref{eq:lifted-path-induced-cut-inequality2} with path $P=(vv_3)$ and $y'_{v_3w}=0$, we obtain
\begin{align*}
y'_{vw}\leq 0\,.    
\end{align*}
\end{itemize}
\end{proof}

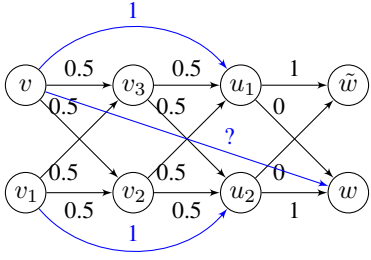
\begin{figure}
\begin{center}
\begin{tikzpicture}[scale=1.3]

\node[style=vertex] (10) at (1, 0) {$v_1$};
\node[style=vertex](20) at (2, 0) {$v_2$};
\node[style=vertex]  (30) at (3, 0) {$u_2$};
\node[style=vertex] (40)  at (4, 0) {${w}$};

\node[style=vertex] (11) at (1, 1) {${v}$};
\node[style=vertex] (21) at (2,1) {$v_3$};
\node[style=vertex] (31) at (3, 1) {$u_1$};
\node[style=vertex] (41) at (4, 1) {$\tilde{w}$};

\draw[base-edge] (10) -- (20)node[midway,below] {\footnotesize{0.5}};
\draw[base-edge] (20) -- (30)node[midway,below] {\footnotesize{0.5}};
\draw [base-edge] (21) -- (30)node[pos=0.3,above]  {\footnotesize{0.5}};
\draw [base-edge](11) -- (21)node[pos=0.5,above]  {\footnotesize{0.5}};
\draw [base-edge](21) -- (31)node[pos=0.5,above]  {\footnotesize{0.5}};
\draw [base-edge](31) -- (41)node[pos=0.5,above]  {\footnotesize{1}};
\draw [base-edge](31) -- (40)node[pos=0.3,above]  {\footnotesize{0}};
\draw [base-edge](20) -- (31)node[pos=0.3,below]  {\footnotesize{0.5}};
\draw[base-edge] (30) -- (41)node[pos=0.3,below]  {\footnotesize{0}};
\draw[base-edge] (11) -- (20)node[pos=0.3,above]  {\footnotesize{0.5}};
\draw[base-edge] (10) -- (21)node[pos=0.3,below]  {\footnotesize{0.5}};
\draw[base-edge] (30) -- (40)node[midway,below]  {\footnotesize{1}};

\draw[lifted-edge] (11) edge [bend left=50] node[midway,above]{\footnotesize{1}}(31);
\draw[lifted-edge] (10) edge [bend right=50] node[midway,above]{\footnotesize{1}}(30);

\draw[lifted-edge] (11) edge node[pos=0.65,above]{\footnotesize{?}} (40);

\end{tikzpicture}
\caption{
Exemplary case where the path-induced cut inequalities~\eqref{eq:path-induced-cut-inequality} fail to give non-trivial upper bounds for lifted edge $y'_{vw}$.
The lifted path-induced cut-inequalities~\eqref{eq:lifted-path-induced-cut-inequality} give the correct upper bound in this case. Example for Proposition~\ref{prop:lifted-path-induced-cut-inequality-strictly-better}.
}
\label{fig:lifted-cut-proof}
\end{center}
\end{figure}

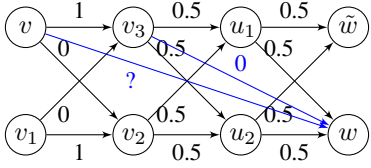
\begin{figure}
\begin{center}
\begin{tikzpicture}[scale=1.3]

\node[style=vertex] (10) at (1, 0) {$v_1$};
\node[style=vertex](20) at (2, 0) {$v_2$};
\node[style=vertex]  (30) at (3, 0) {$u_2$};
\node[style=vertex] (40)  at (4, 0) {${w}$};

\node[style=vertex] (11) at (1, 1) {${v}$};
\node[style=vertex] (21) at (2,1) {$v_3$};
\node[style=vertex] (31) at (3, 1) {$u_1$};
\node[style=vertex] (41) at (4, 1) {$\tilde{w}$};

\draw[base-edge] (10) -- (20)node[midway,below] {\footnotesize{1}};
\draw[base-edge] (20) -- (30)node[midway,below] {\footnotesize{0.5}};
\draw [base-edge] (21) -- (30)node[pos=0.3,above]  {\footnotesize{0.5}};
\draw [base-edge](11) -- (21)node[pos=0.5,above]  {\footnotesize{1}};
\draw [base-edge](21) -- (31)node[pos=0.5,above]  {\footnotesize{0.5}};
\draw [base-edge](31) -- (41)node[pos=0.5,above]  {\footnotesize{0.5}};
\draw [base-edge](31) -- (40)node[pos=0.3,above]  {\footnotesize{0.5}};
\draw [base-edge](20) -- (31)node[pos=0.3,below]  {\footnotesize{0.5}};
\draw[base-edge] (30) -- (41)node[pos=0.3,below]  {\footnotesize{0.5}};
\draw[base-edge] (11) -- (20)node[pos=0.3,above]  {\footnotesize{0}};
\draw[base-edge] (10) -- (21)node[pos=0.3,below]  {\footnotesize{0}};
\draw[base-edge] (30) -- (40)node[midway,below]  {\footnotesize{0.5}};

\draw[lifted-edge] (21) edge node[midway,above]{\footnotesize{0}}(40);

\draw[lifted-edge] (11) edge node[pos=0.3,below]{\footnotesize{?}} (40);

\end{tikzpicture}
\caption{
Exemplary failure case for the lifted path-induced cut inequalities~\eqref{eq:lifted-path-induced-cut-inequality}.
The lifted path-induced cut inequalities~\eqref{eq:lifted-path-induced-cut-inequality2} give the correct upper bound for lifted edge $y'_{vw}$. Example for Proposition~\ref{prop:lifted-path-induced-cut-inequality-strictly-better}.
}
\label{fig:lifted-cut-proof2}
\end{center}
\end{figure}

\subsection{Symmetric Form of Cut Inequalities}
\label{sec:symmetric-inequalities}

Inequalities symmetric to \eqref{eq:path-induced-cut-inequality}:
\begin{gather}
  \nonumber   \forall vw\in E'\ \forall P\in uw\mhyphen\text{paths}(G)\text{ s.t. }\ vu \in \mathcal{R} \wedge u\neq v:\\
 y'_{vw}\leq \sum_{i\in P_V} \sum_{\substack{k\notin P_V,\\ vk \in \mathcal{R}}}y_{ki} \,.
\label{eq:path-induced-cut-inequality-sym}
\end{gather}

Inequalities symmetric to \eqref{eq:lifted-path-induced-cut-inequality}
\begin{gather}
  \nonumber   \forall vw\in E'\ \forall P\in uw\mhyphen\text{paths}(G \cup G')\text{ s.t. }\ vu\in \mathcal{R} \wedge u \neq v: \\
  \nonumber y'_{vw}\leq \sum_{i\in P_V} \sum_{\substack{k\notin P_V,\\ vk \in \mathcal{R}}}y_{ki}-\sum_{ij\in P_{E'}}y'_{ij}\\
 +\sum_{ij\in P_{E'}\cap E}y_{ij} \,.
\label{eq:lifted-path-induced-cut-inequality-sym}
\end{gather}

Inequalities symmetric to \eqref{eq:lifted-path-induced-cut-inequality2}
\begin{gather}
 \nonumber \forall vw\in E'\ \forall P\in uw\mhyphen\text{paths }(G\cup G')\ \text{ s.t. } vu\in E':\\
 \nonumber y'_{vw}\leq \sum_{i\in P_V\setminus u} \ \sum_{\substack{k\notin P_V, \\ vk \in \mathcal{R}}} y_{ki}  -\sum_{ij\in P_{E'}}y'_{ij}\\
   +\sum_{ij\in P_{E'}\cap E}y_{ij}+y'_{vu} \,.
  \label{eq:lifted-path-induced-cut-inequality2-sym} 
\end{gather}

\begin{restatable}{prop}{liftedpathinducedcutbetter-sym}
\label{prop:lifted-path-induced-cut-inequality-strictly-better-sym}
    The lifted path-induced cut inequalities~\eqref{eq:lifted-path-induced-cut-inequality-sym} define a strictly tighter relaxation than the path-induced cut inequalities~\eqref{eq:path-induced-cut-inequality-sym}.\\
    The lifted path-induced cut inequalities~\eqref{eq:lifted-path-induced-cut-inequality-sym} and~\eqref{eq:lifted-path-induced-cut-inequality2-sym} define a strictly better relaxation than~\eqref{eq:lifted-path-induced-cut-inequality-sym} alone.
\end{restatable}

\begin{proof}
Analogical to the proof of Proposition~\ref{prop:lifted-path-induced-cut-inequality-strictly-better}. See Figure~\ref{fig:lifted-cut-sym-proof} for example analogical to the one in Figure~\ref{fig:lifted-cut-proof} and Figure~\ref{fig:lifted-cut-sym-proof2} for example analogical to the one in Figure~\ref{fig:lifted-cut-proof2}.
\end{proof}

\begin{restatable}{prop}{symmetric-different}
\label{prop:sym-different}

\begin{enumerate}
    \item The path-induced cut inequalities~\eqref{eq:path-induced-cut-inequality} together with their symmetric counterpart \eqref{eq:path-induced-cut-inequality-sym} define a strictly tighter relaxation than  inequalities~\eqref{eq:path-induced-cut-inequality} alone.
    \item The path-induced cut inequalities~\eqref{eq:lifted-path-induced-cut-inequality} together with their symmetric counterpart \eqref{eq:lifted-path-induced-cut-inequality-sym} define a strictly tighter relaxation than  inequalities~\eqref{eq:lifted-path-induced-cut-inequality} alone.
       \item Using path-induced cut inequalities~\eqref{eq:lifted-path-induced-cut-inequality2-sym} together with \eqref{eq:lifted-path-induced-cut-inequality}, \eqref{eq:lifted-path-induced-cut-inequality2} and \eqref{eq:lifted-path-induced-cut-inequality-sym} strictly improves the relaxation.
\end{enumerate}    
\end{restatable}

\begin{proof}
\begin{enumerate}
    \item See the example in Figure \ref{fig:sym-cut-proof}.\\
    Upper bound  on $y'_{vw}$ by \eqref{eq:path-induced-cut-inequality}:
    $y'_{vw}\leq 0.5 + 0.5=1$.\\
    Upper bound  on $y'_{vw}$ by \eqref{eq:path-induced-cut-inequality-sym}: $y'_{vw}\leq 0$.
    
    \item See the example in Figure~\ref{fig:lifted-cut-sym-proof}.\\
    Upper bound  on $y'_{vw}$ by~\eqref{eq:lifted-path-induced-cut-inequality}: 
    $y'_{vw}\leq 0.5 + 0.5=1$.\\
    Upper bound  on $y'_{vw}$ by~\eqref{eq:lifted-path-induced-cut-inequality-sym} using path $P=(u_2w)$: $y'_{vw}\leq 0+0.5+0.5-1=0$.
    
    \item See the example in Figure \ref{fig:lifted-cut-sym-proof2}.\\
    Upper bounds  on $y'_{vw}$ by \eqref{eq:lifted-path-induced-cut-inequality}, \eqref{eq:lifted-path-induced-cut-inequality2}, \eqref{eq:lifted-path-induced-cut-inequality-sym}: $y'_{vw}\leq 1$.
    Upper bound  on $y'_{vw}$ by \eqref{eq:lifted-path-induced-cut-inequality2-sym} using path $P=(uw)$ and $y'_{vu}=0$: $y'_{vw}\leq 0$.
\end{enumerate}

\end{proof}

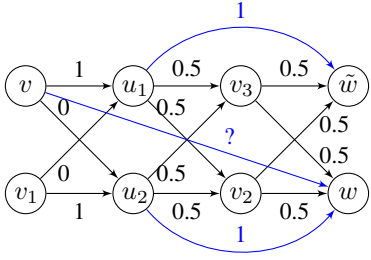
\begin{figure}
\begin{center}
\begin{tikzpicture}[scale=1.3]

\node[style=vertex] (10) at (1, 0) {$v_1$};
\node[style=vertex](20) at (2, 0) {$u_2$};
\node[style=vertex]  (30) at (3, 0) {$v_2$};
\node[style=vertex] (40)  at (4, 0) {${w}$};

\node[style=vertex] (11) at (1, 1) {${v}$};
\node[style=vertex] (21) at (2,1) {$u_1$};
\node[style=vertex] (31) at (3, 1) {$v_3$};
\node[style=vertex] (41) at (4, 1) {$\tilde{w}$};

\draw[base-edge] (10) -- (20)node[midway,below] {\footnotesize{1}};
\draw[base-edge] (20) -- (30)node[midway,below] {\footnotesize{0.5}};
\draw [base-edge] (21) -- (30)node[pos=0.3,above]  {\footnotesize{0.5}};
\draw [base-edge](11) -- (21)node[pos=0.5,above]  {\footnotesize{1}};
\draw [base-edge](21) -- (31)node[pos=0.5,above]  {\footnotesize{0.5}};
\draw [base-edge](31) -- (41)node[pos=0.5,above]  {\footnotesize{0.5}};
\draw [base-edge](31) -- (40)node[pos=0.7,right]  {\footnotesize{0.5}};
\draw [base-edge](20) -- (31)node[pos=0.3,below]  {\footnotesize{0.5}};
\draw[base-edge] (30) -- (41)node[pos=0.7,right]  {\footnotesize{0.5}};
\draw[base-edge] (11) -- (20)node[pos=0.3,above]  {\footnotesize{0}};
\draw[base-edge] (10) -- (21)node[pos=0.3,below]  {\footnotesize{0}};
\draw[base-edge] (30) -- (40)node[midway,below]  {\footnotesize{0.5}};

\draw[lifted-edge] (21) edge [bend left=50] node[midway,above]{\footnotesize{1}}(41);
\draw[lifted-edge] (20) edge [bend right=50] node[midway,above]{\footnotesize{1}}(40);

\draw[lifted-edge] (11) edge node[pos=0.65,above]{\footnotesize{?}} (40);

\end{tikzpicture}
\caption{The best upper bound on $y'_{vw}$ is provided by inequalities~\eqref{eq:lifted-path-induced-cut-inequality-sym}. Example for Proposition~\ref{prop:lifted-path-induced-cut-inequality-strictly-better-sym} and Proposition~\ref{prop:sym-different}.}
\label{fig:lifted-cut-sym-proof}
\end{center}
\end{figure}

\begin{figure}
\begin{center}
\begin{tikzpicture}[scale=1.3]

\node[style=vertex] (10) at (1, 0) {$v_1$};
\node[style=vertex](20) at (2, 0) {$v_2$};
\node[style=vertex]  (30) at (3, 0) {$u$};
\node[style=vertex] (40)  at (4, 0) {${w}$};

\node[style=vertex] (11) at (1, 1) {${v}$};
\node[style=vertex] (21) at (2,1) {$v_3$};
\node[style=vertex] (31) at (3, 1) {$v_4$};
\node[style=vertex] (41) at (4, 1) {$\tilde{w}$};

\draw[base-edge] (10) -- (20)node[midway,below] {\footnotesize{0.5}};
\draw[base-edge] (20) -- (30)node[midway,below] {\footnotesize{0.5}};
\draw [base-edge] (21) -- (30)node[pos=0.3,above]  {\footnotesize{0.5}};
\draw [base-edge](11) -- (21)node[pos=0.5,above]  {\footnotesize{0.5}};
\draw [base-edge](21) -- (31)node[pos=0.5,above]  {\footnotesize{0.5}};
\draw [base-edge](31) -- (41)node[pos=0.5,above]  {\footnotesize{1}};
\draw [base-edge](31) -- (40)node[pos=0.7,right]  {\footnotesize{0}};
\draw [base-edge](20) -- (31)node[pos=0.3,below]  {\footnotesize{0.5}};
\draw[base-edge] (30) -- (41)node[pos=0.7,right]  {\footnotesize{0}};
\draw[base-edge] (11) -- (20)node[pos=0.3,left]  {\footnotesize{0.5}};
\draw[base-edge] (10) -- (21)node[pos=0.3,left]  {\footnotesize{0.5}};
\draw[base-edge] (30) -- (40)node[midway,below]  {\footnotesize{1}};

\draw[lifted-edge] (11) edge node[pos=0.4,below]{\footnotesize{0}}(30);

\draw[lifted-edge] (11) edge node[pos=0.65,above]{\footnotesize{?}} (40);

\end{tikzpicture}
\caption{The best upper bound on $y'_{vw}$ is provided by inequalities~\eqref{eq:lifted-path-induced-cut-inequality2-sym}. Example for Proposition~\ref{prop:lifted-path-induced-cut-inequality-strictly-better-sym} and Proposition~\ref{prop:sym-different}.}
\label{fig:lifted-cut-sym-proof2}
\end{center}
\end{figure}
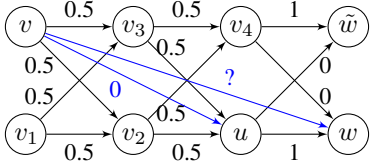

\begin{figure}
\begin{center}
\begin{tikzpicture}

\node[style=vertex] (11) at (1, 0.5) {$v$};
\node[style=vertex] (10) at (1, -1.5) {$\tilde{v}$};

\node[style=vertex] (22) at (2, 0.5) {$v_3$};
\node[style=vertex] (21) at (2, -0.5) {$v_2$};
\node[style=vertex] (20) at (2, -1.5) {$v_1$};

\node[style=vertex] (31) at (3, 0.5) {$v_5$};
\node[style=vertex] (30) at (3, -1.5) {$v_4$};

\node[style=vertex] (42) at (4, 0.5) {$v_8$};
\node[style=vertex] (41) at (4, -0.5) {$v_7$};
\node[style=vertex] (40) at (4, -1.5) {$v_6$};

\node[style=vertex] (51) at (5, 0.5) {$\tilde{w}$};
\node[style=vertex] (50) at (5, -1.5) {$w$};

\draw[base-edge] (10) edge node[midway,below]{\footnotesize{1}}(20);
\draw[base-edge] (10) edge node[midway,left]{\footnotesize{0}}(21);

\draw[base-edge] (20) edge node[midway,below]{\footnotesize{1}}(30);
\draw[base-edge] (21) edge node[pos=0.5,left]{\footnotesize{0}}(30);

\draw[base-edge] (30) edge node[midway,below]{\footnotesize{1}}(40);
\draw[base-edge] (30) edge node[midway,left]{\footnotesize{0}}(41);

\draw[base-edge] (40) edge node[midway,below]{\footnotesize{1}}(50);
\draw[base-edge] (41) edge node[midway,above]{\footnotesize{0}}(50);

\draw[base-edge] (11) edge node[midway,above]{\footnotesize{0.5}}(22);
\draw[base-edge] (11) edge node[pos=0.5,left]{\footnotesize{0.5}}(21);

\draw[base-edge] (22) edge node[midway,above]{\footnotesize{0.5}}(31);
\draw[base-edge] (21) edge node[midway,right]{\footnotesize{0.5}}(31);

\draw[base-edge] (31) edge node[midway,right]{\footnotesize{0.5}}(41);
\draw[base-edge] (31) edge node[midway,above]{\footnotesize{0.5}}(42);

\draw[base-edge] (41) edge node[pos=0.5,right]{\footnotesize{0.5}}(51);
\draw[base-edge] (42) edge node[midway,above]{\footnotesize{0.5}}(51);

\draw[base-edge] (22) edge node[pos=0.8,right]{\footnotesize{0}}(30);
\draw[base-edge] (42) edge node[pos=0.8,right]{\footnotesize{0}}(50);

\draw[lifted-edge] (11) edge node[midway,above]{\footnotesize{?}}(50) ;

\end{tikzpicture}
\caption{
The best upper bound on $y'_{vw}$ is provided by inequalities~\eqref{eq:path-induced-cut-inequality-sym}. Example for Proposition~\ref{prop:sym-different}.
}
\label{fig:sym-cut-proof}
\end{center}
\end{figure}
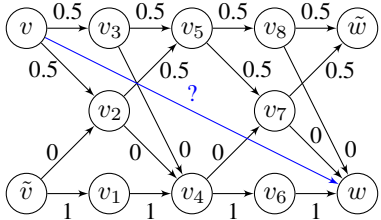

\subsection{Other Valid Inequalities}
Basic flow constraints~\eqref{eq:flow-conservation} together with the advanced constrains on lifted edges~\eqref{eq:vwcut1}-\eqref{eq:lifted-path-induced-cut-inequality2} are sufficient for defining the set of feasible solutions of the lifted disjoint paths problem~\eqref{eq:lifted-disjoint-paths-problem}. Moreover, they define an efficient LP relaxation (Section~\ref{sec:constraints}) and enable efficient separation procedures (Section~\ref{sec:separation}). 
Below, we present lifted flow inequalities specific to the lifted disjoint paths problem applied to MOT that help to improve the speed of our ILP solver.
The inequalities depend on the fact that every node can be connected to maximally one node in each time frame.
Therefore the number of lifted edges originating (or ending) in a given point and ending (resp.\ originating) in a specific time frame is at most one.
\begin{align}
\nonumber\forall k,l \in \{1,\dots,T\}: k>l,\ \forall v\in V_l:\\
\label{eq:lifted-flow1}    \sum_{vu\in E':u\in V_k} y'_{vu}\leq x_v\,,
\end{align}
\begin{align}
\nonumber\forall k,l \in \{1,\dots,T\}: k<l,\ \forall w\in V_l:\\
\label{eq:lifted-flow2}    \sum_{uw\in E':u\in V_k} y'_{uw}\leq x_w\,.
\end{align}
The number of constraints~\eqref{eq:lifted-flow1} and~\eqref{eq:lifted-flow2} is linear in the number of vertices. Therefore, we add them to our initial constraint set. This enables to reduce the search space for the branch and bound method in the early solver stages when only few constraints of type~\eqref{eq:path-inequalities}-\eqref{eq:lifted-path-induced-cut-inequality2} have been added.

\subsection{Proofs for Section \ref{sec:complexity} Complexity}
\label{sec:appendix-complexity}
We define $Y_{GG'}$ to be the set of all $(y,y')\in \{0,1\}^E\times \{0,1\}^{E'}$ such that $(y,y')$ are feasible solutions of the lifted disjoint path problem~\eqref{eq:lifted-disjoint-paths-problem}.
\paragraph{Integer multicommodity flow.}

The integer multicommodity flow problem is defined on a directed graph $\mathcal{G} = \mathcal{(V,E)}$ with edge capacities $c \in \N^\mathcal{E}$ and source/sink pairs $s_i t_i$ and edge flows $f_i \in \N^{\mathcal{E}}$ and demands $R_i$, $i=1,\ldots,k$.

The aim is to send $k$ flows from their sources to their sinks such that the flows obey the edge capacities. Formally,
\begin{align}
    & \sum_{i=1}^k f^i_e \leq c_e & \forall e \in E \\
    & \sum_{u: uv \in E} f^i_{uv} = \sum_{w: vw \in E} f^i_{vw} &\forall i\in[k]\  \forall v \notin \{s_i, t_i\} \\
 \label{eq:flow-tr-req}   & \sum_{v: s_i v \in E} f^i_{s_i v} \geq R_i &\forall i\in[k]
\end{align}
where $[k]$ denotes the set $\{1,\dots,k\}$.
Even has shown in~\cite{EvenMulti} that the integer multicommodity flow problem is NP-complete also in the case of unit capacity edges and two source sink pairs.
Below we detail a construction that gives us a correspondence between edge-disjoint paths in $\mathcal{G}$  and node-disjoint paths in the transformed graph $G$.
This construction is similar to transforming a graph into its line graph.
The lifted edges in the transformed graph will count how many units of flow go from sources to sinks.

\begin{lemma}
\label{lemma:multicommodity-flow-reduction}

There exists a polynomial transformation from any graph $\mathcal{G}$ with source/sink pairs $s_i,t_i$, $i=1,\ldots,k$ with demands $R_i$ to
a pair of graphs $G$ and $G'$ with edge costs $c$ and $c'$ respectively such that
there exists a feasible integer multicommodity flow in $\mathcal{G}$ if and only if the lifted disjoint paths problem for $G,G'$ has objective
$\min_{(y,y')\in Y_{GG'}} \la c, y \ra + \la c', y' \ra \leq - \sum_{i=1}^k R_i$.
\end{lemma}

\begin{proof}
Without loss of generality, we consider these feasible flow sets $f_1,\dots,f_k$ where it holds $\forall i\in [k]:$ $\sum_{s_i v \in E} f^i_{s_i v} = R_i$. Note that if the flow of commodity $i$ is higher than its demand $R_i$, we can reduce it to $R_i$ by removing the flow across one or more $s_it_i\mhyphen$paths in $\mathcal{G}$ without violating other constraints.\\
We first detail the graph transformation (see Figures~\ref{fig:mcf-reduction} and \ref{fig:mcf-reduction2}).
\begin{itemize}
 \item For all edges $ij\in \mathcal{E}$ add a vertex $v_{ij}$ to $V$.
 \item For each pair of vertices  $v_{ij},v_{jk}\in V$ add an edge $(v_{ij},v_{jk})$ to $E$.
 \item Add vertices $s$ and $t$ to $V$.
 \item Add to $V$ vertices $s_{i}^1,s_{i}^2,\dots ,s_{i}^{R_i}$ representing requirements of each commodity $i$.
 \item For each vertex $s_{i}^r$ add an edge $(s,s_{i}^r)$ to $E$.
 \item For each pair of vertices $s_{i}^r,v_{s_ij}$ add edge $(s_{i}^r,v_{s_ij})$ to $E$.
 \item For all $v_{kt_i}\in V$ (representing and edge from $k$ to $t_i$ in $\mathcal{G}$) add an edge $(v_{kt_i},t)$ to $E$.
 \item For all pairs of vertices $v_{s_ij}$ $v_{kt_i}\in V$ add an edge $(v_{s_ij},v_{kt_i})$ to $E'$. That is, the lifted edges connect all vertices representing edges from $s_i$ in $\mathcal{G}$ with vertices representing the edges to $t_i$ in $\mathcal{G}$.
 \item Cost function on base edges $\forall e\in E:$  $c_{e}=0$. 
 \item Cost function on lifted edges $\forall e'\in E':$ $c'_{e'}=-1$.
\end{itemize}
An illustration of this construction can be seen in 
Figures~\ref{fig:mcf-reduction} and \ref{fig:mcf-reduction2}.
Note that the construction of $\mathcal{G}$ in \cite{EvenMulti} allows $s_i=s_j$ for $i \neq j$.
In this case, we still construct separate vertices for their incident edges in $G$.

Every path $\mathcal{P}=(s_ik_1,k_1k_2,\dots, k_nt_i)$ in $\mathcal{G}$ can be assigned to a path $P=(s{s_{i}^r},{s_{i}^r}v_{s_ik_1},v_{s_ik_1}v_{k_1k_2},\dots,v_{k_nt_i}t)$ in $G$ where $r\in[R_i]$ can be chosen arbitrarily and vice versa. Note that such a path $P$ saturates exactly one lifted edge $(v_{s_ik_1},v_{k_nt_i})$.
Moreover, every feasible set of flow functions $f_1,\dots, f_k$ satisfying  for all $i\in [k]:$ $\sum_{s_i v \in E} f^i_{s_i v} = R_i $ defines a set of edge-disjoint paths from $s_1,\dots, s_k$ to $t_1,\dots,t_k$ in $\mathcal{G}$.
This set corresponds to a set of $\sum_{i=1}^k R_i$ $st$-paths in $G$ whose edges and vertices are disjoint and where every path saturates exactly one lifted edge $v_{s_ij}v_{kt_i}$.  Every lifted edge contributes with $-1$ to the total cost. So, this set of disjoint $st$-paths has total cost $-\sum_{i=1}^k R_i$.\\
Reversely, let us have a set of vertex- and edge-disjoint $st$-paths in $G$ of size $\sum_{i=1}^k R_i$ where every path contains some $v_{s_{i}j}v_{kt_i}$-path as its subpath and therefore its cost is $ -\sum_{i=1}^k R_i$. This set defines uniquely a set of feasible flow functions $f_1,\dots,f_k$.\\
So, there exist feasible functions $f_1,\dots,f_k$ 
 satisfying $f_i=R_i$ for all $i\in [k]$
 iff $\min\limits_{(y,y')\in Y_{GG'}}\gamma(y,y')\leq -\sum_{i=1}^k R_i$.

\end{proof}

\begin{figure}
\begin{center}
\begin{tikzpicture}[scale=1.3]

\node[style=vertex] (10) at (1, 0) {$s_1$};
\node[style=vertex] (11) at (1, 1) {$s_2$};
\node[style=vertex](20) at (2, 0) {$a$};
\node[style=vertex] (21) at (2,0.5) {$b$};
\node[style=vertex] (22) at (2,1) {$c$};

\node[style=vertex]  (30) at (3, 0.5) {$d$};
\node[style=vertex] (31) at (3, 1) {$e$};

\node[style=vertex] (40)  at (4, 0) {$t_1$};
\node[style=vertex] (41) at (4, 1) {$t_2$};

\draw[base-edge] (10) -- (20)node[midway,below] {};

\draw [base-edge] (21) -- (30)node[pos=0.3,above]  {};
\draw [base-edge](11) -- (21)node[pos=0.5,above]  {};
\draw [base-edge](21) -- (31)node[pos=0.5,above]  {};
\draw [base-edge](31) -- (41)node[pos=0.5,above]  {};
\draw [base-edge](31) -- (40)node[pos=0.3,above]  {};
\draw [base-edge](22) -- (31)node[pos=0.3,below]  {};
\draw[base-edge] (30) -- (41)node[pos=0.3,below]  {};
\draw[base-edge] (11) -- (22)node[pos=0.3,above]  {};
\draw[base-edge] (10) -- (21)node[pos=0.3,below]  {};
\draw[base-edge] (20) -- (40)node[midway,below] {};


\end{tikzpicture}
\caption{Integer multicommodity flow network transformation: Original graph.}
\label{fig:mcf-reduction}
\end{center}
\end{figure}
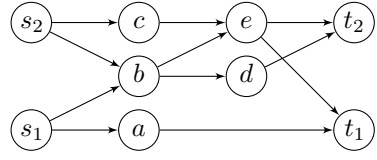

\begin{figure}
\begin{center}
\tikzset{vertex/.style={circle, draw, inner sep=0pt,text height=3.4mm,text width=3.6mm,text depth=1.5mm,align=center}}

\begin{tikzpicture}[scale=1.3]

\node[style=vertex] (01) at (0, 0) {\footnotesize{$s$}};
\node[style=vertex] (51) at (5, 0) {\footnotesize{$t$}};

\node[style=vertex] (12) at (1, -2) {\footnotesize{$s_{1}^2$}};
\node[style=vertex] (13) at (1, -1) {\footnotesize{$s_{1}^3$}};
\node[style=vertex] (14) at (1, 1) {\footnotesize{$s_{2}^1$}};
\node[style=vertex] (15) at (1, 2) {\footnotesize{$s_{2}^2$}};

\node[style=vertex] (21) at (2, -2) {\footnotesize{$s_{1}a$}};

\node[style=vertex] (23) at (2, -1) {\footnotesize{$s_{1}b$}};
\node[style=vertex] (24) at (2, 1) {\footnotesize{$s_{2}b$}};
\node[style=vertex] (25) at (2, 2) {\footnotesize{$s_{2}c$}};

\node[style=vertex] (31) at (4, -2) {\footnotesize{$at_{1}$}};

\node[style=vertex] (33) at (3, 0) {\footnotesize{$bd$}};
\node[style=vertex] (34) at (3, 1) {\footnotesize{$be$}};
\node[style=vertex] (35) at (3, 2) {\footnotesize{$ce$}};

\node[style=vertex] (42) at (4, 1) {\footnotesize{$dt_2$}};
\node[style=vertex] (43) at (4, -1) {\footnotesize{$et_1$}};
\node[style=vertex] (44) at (4, 2) {\footnotesize{$et_2$}};

\draw[base-edge] (01) edge (12);
\draw[base-edge] (01) edge (13);
\draw[base-edge] (01) edge (14);
\draw[base-edge] (01) edge (15);

\draw[base-edge] (12) edge (21);

\draw[base-edge] (12) edge (23);

\draw[base-edge] (13) edge (21);

\draw[base-edge] (13) edge (23);

\draw[base-edge] (14) edge (24);
\draw[base-edge] (14) edge (25);

\draw[base-edge] (15) edge (24);
\draw[base-edge] (15) edge (25);

\draw[base-edge] (21) edge (31);

\draw[base-edge] (23) edge (33);
\draw[base-edge] (23) edge (34);

\draw[base-edge] (24) edge (33);
\draw[base-edge] (24) edge (34);

\draw[base-edge] (25) edge (35);

\draw[base-edge] (31) edge [bend right=30] (51);

\draw[base-edge] (33) edge (42);

\draw[base-edge] (33) edge (42);

\draw[base-edge] (34) edge (43);
\draw[base-edge] (34) edge (44);

\draw[base-edge] (35) edge (43);
\draw[base-edge] (35) edge (44);

\draw[base-edge] (42) edge (51);
\draw[base-edge] (43) edge (51);
\draw[base-edge] (44) edge (51);

\draw[lifted-edge] (21) edge[bend right=30] node[pos=0.3,above]{\footnotesize{-1}} (31);
\draw[lifted-edge] (21) edge  node[pos=0.3,above]{\footnotesize{-1}}  (43);

\draw[lifted-edge] (23) edge node[pos=0.3,above]{\footnotesize{-1}} (31);
\draw[lifted-edge] (23) edge[bend left=30] node[pos=0.3,above]{\footnotesize{-1}} (43);

\draw[lifted-edge] (24) edge [bend right=30]node[pos=0.3,above]{\footnotesize{-1}} (42);
\draw[lifted-edge] (24) edge node[pos=0.3,above]{\footnotesize{-1}} (44);

\draw[lifted-edge] (25) edge node[pos=0.3,above]{\footnotesize{-1}} (42);
\draw[lifted-edge] (25) edge[bend left=30] node[pos=0.3,above]{\footnotesize{-1}} (44);

\end{tikzpicture} 
\caption{Integer multicommodity flow network transformation. Transformed graph from Figure \ref{fig:mcf-reduction} for flow demands $R_1=2, R_2=2$. Edges without label have cost 0.}
\label{fig:mcf-reduction2}
\end{center}
\end{figure}
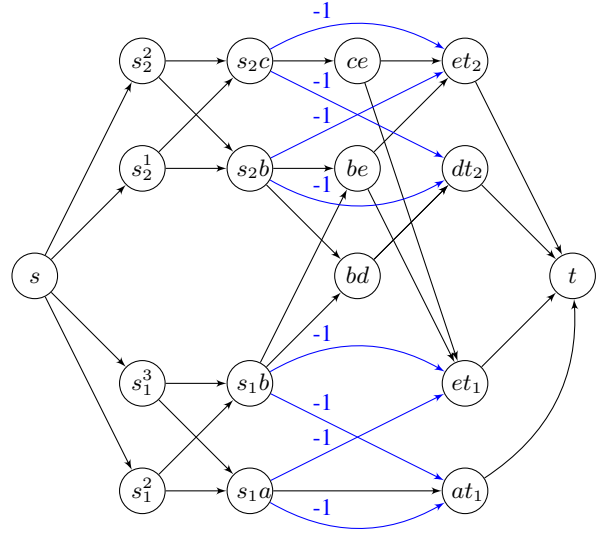

\multicommodityflowreduction*
\begin{proof}
The NP-complete integer multicommodity flow problem with unit edge capacities can be reduced in polynomial time to the lifted disjoint paths problem~\eqref{eq:lifted-disjoint-paths-problem} with negative lifted edges only. The transformation is described in Lemma~\ref{lemma:multicommodity-flow-reduction}.

\end{proof}

\paragraph{3-SAT.}
The boolean satisfiability problem (SAT) is a classical NP-complete problem~\cite{cook1971complexity}. 
A transformation from its NP-complete special version 3-SAT is commonly used for proving than a problem is NP-hard or NP-complete.

\threesatreduction*

\begin{proof}

Below, we detail a transformation from 3-SAT to the lifted disjoint paths problem with positive lifted edges only. For the transformation, it holds that a 3-SAT formula consisting of $k$ clauses has a true assignment iff $\min\limits_{(y,y')\in Y_{GG'}}\gamma(y,y')\leq -(k-1)$.

Let a 3-SAT problem containing $k$ ordered clauses $C_1\dots C_k$ be given.
Each clause $C_i$ consists of a conjunction of literals, which is either a variable $a$ or its complement $\overline{a}$.
We construct graphs $G=(V,E)$ and $G'=(V',E')$ as follows. 
\begin{itemize}
 \item The graph $G$ has $k$ layers. Every layer corresponds to one clause. Each layer contains 3 vertices labeled with the literals in the corresponding clause.
 Specifically, for a variable $a$ in clause $C_i$ we associate node $v_{ia}$, analoguously for a complemented variable $\overline{b}$ in clause $C_i$ we associate node $v_{i\bar{b}}$.
 \item For every pair of vertices $v_{il_1}\in V$ and $v_{i+1l_2}\in V$ where $l_1\neq \bar{l_2}$ add an edge $(v_{il_1},v_{i+1l_2})$ to $E$ and set $c_{(v_{il_1},v_{i+1l_2})}=-1$.
 \item For every variable $a$ and every pair of vertices $v_{ia},v_{j\bar{a}}\in V$ where $j>i+1$ add an edge $(v_{ia},v_{j\bar{a}})$ to $E'$ and set $c'_{(v_{ia},v_{j\bar{a}})}=k$.
 Do so analoguously for every pair of variables $v_{i\bar{a}}$ and $v_{ja}$.
 \item Add an edge from $s$ to all vertices corresponding to the first clause.
 And an edge to $t$ from all vertices corresponding to the last clause.
\end{itemize}
An illustration of this construction can be found in Figure~\ref{fig:3sat}.

\begin{figure}
\begin{tikzpicture}[scale=1.3]

\node[style=vertex] (10) at (1, 0) {$\bar{c}$};

\node[style=vertex](20) at (2, 0) {$\bar{d}$};
\node[style=vertex]  (30) at (3, 0) {$e$};
\node[style=vertex] (40)  at (4, 0) {$\bar{e}$};

\node[style=vertex] (11) at (1, 1) {$b$};
\node[style=vertex] (21) at (2,1) {$c$};
\node[style=vertex] (31) at (3, 1) {$c$};
\node[style=vertex] (41) at (4, 1) {$c$};

\node[style=vertex] (12) at (1, 2) {$a$};
\node[style=vertex] (22) at (2, 2) {$a$};
\node[style=vertex] (32) at (3, 2) {$\bar{a}$};
\node[style=vertex] (42) at (4, 2) {$\bar{a}$};

\node[style=vertex] (01) at (0, 1) {$s$};
\node[style=vertex] (51) at (5, 1) {$t$};

\draw[base-edge] (10) -- (20)node[midway,below] {\footnotesize{-1}};
\draw[base-edge] (20) -- (30)node[midway,below] {\footnotesize{-1}};
\draw [lifted-edge] (10) -- (41)node[pos=0.7,below] {\footnotesize{4}};
\draw[lifted-edge] (10) -- (31)node[midway,above] {\footnotesize{4}} ;
\draw [base-edge] (21) -- (32) node[pos=0.6,above]  {\footnotesize{-1}};
\draw [base-edge] (12) -- (21)node[near start,above]  {\footnotesize{-1}};
\draw [base-edge] (21) -- (30)node[pos=0.7,below]  {\footnotesize{-1}};
\draw[base-edge] (12) -- (22)node[midway,above]  {\footnotesize{-1}};
\draw[base-edge] (32) -- (42)node[midway,above]  {\footnotesize{-1}};
\draw [base-edge](11) -- (21)node[pos=0.2,above]  {\footnotesize{-1}};
\draw [base-edge](21) -- (31)node[pos=0.2,above]  {\footnotesize{-1}};
\draw [base-edge](31) -- (41)node[pos=0.2,above]  {\footnotesize{-1}};
\draw [base-edge](20) -- (31)node[pos=0.3,below]  {\footnotesize{-1}};
\draw [base-edge](31) -- (42)node[pos=0.6,above]  {\footnotesize{-1}};
\draw[base-edge] (30) -- (41)node[pos=0.3,below]  {\footnotesize{-1}};
\draw[base-edge] (11) -- (22)node[pos=0.6,above]  {\footnotesize{-1}};
\draw[base-edge] (11) -- (20)node[midway,above]  {\footnotesize{-1}};
\draw [base-edge](32) -- (41)node[pos=0.2,above]  {\footnotesize{-1}};
\draw[base-edge] (22) -- (31)node[pos=0.2,above]  {\footnotesize{-1}};
\draw[base-edge] (31) -- (40)node[midway,above]  {\footnotesize{-1}};
\draw [base-edge](10) -- (22)node[pos=0.9,below]  {\footnotesize{-1}};

\draw [base-edge](12) -- (20)node[pos=0.1,below]  {\footnotesize{-1}};
\draw [base-edge](22) -- (30)node[pos=0.1,below]  {\footnotesize{-1}};
\draw [base-edge](20) -- (32)node[pos=0.9,below]  {\footnotesize{-1}};
\draw [base-edge](32) -- (40)node[pos=0.1,below]  {\footnotesize{-1}};
\draw [base-edge](30) -- (42)node[pos=0.9,below]  {\footnotesize{-1}};

\draw[base-edge] (01) -- (10);
\draw[base-edge] (01) -- (11);
\draw[base-edge] (01) -- (12);

\draw [base-edge](40) --  (51) ;
\draw [base-edge] (41) -- (51);
\draw [base-edge] (42) -- (51);


\draw[lifted-edge] (12) edge [bend left=40]  node[midway,above]{\footnotesize{4}} (32);
\draw[lifted-edge] (22) edge [bend left=40] node[midway,above]{\footnotesize{4}} (42); 
\draw[lifted-edge] (12) edge [bend left=60]  node[midway,above]{\footnotesize{4}} (42);

\end{tikzpicture} 
\label{fig:multi-com-flow-trafo}
\caption{
Reduction to lifted disjoint paths problem for 3-SAT formula
$(a\vee b\vee \bar{c})\wedge (a\vee c\vee \bar{d})\wedge(\bar{a}\vee c\vee e)\wedge(\bar{a}\vee c\vee \bar{e})$.
}
\label{fig:3sat}
\end{figure}
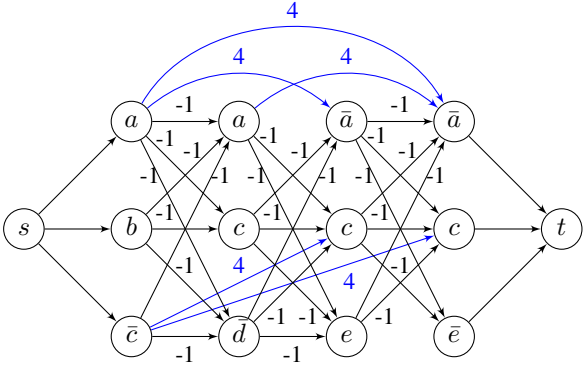

 Every path $P\in st$-paths$(G)$ that has cost $-(k-1)$ saturates vertices labelled by non-contradicting literals. We can obtain a 3-SAT solution from P as follows. If $v_{ia}\in P_V$, set variable $a:=true$. If $v_{j\bar{b}}\in P$, set variable $b:=false$. Variables not contained as labels of vertices in $P_V$ can have arbitrary values. \\
 Similarly, every solution of 3-SAT problem defines at least one path $P\in st$-paths$(G)$ that has cost $-(k-1)$.

\end{proof}
\subsection{Implementation Details on the Lifted Disjoint Paths Solver}
The solver for the lifted disjoint paths problem is implemented in C++ and builds upon Gurobi 7.5.
All experiments were conducted on a machine with a 6-Core Intel 2.00GHz CPU and 128 GB RAM. 

\subsection{Optimal data association}
\label{sec:optimal_assignment_definition}
The experiment of Section \ref{sec:exp_long_range_edges} compares the assignments of our tracking system with the optimal assignments. 
We elaborate on the details to obtain the optimal assignments. 
We start with the pre-processed input detections, according to Section \ref{sec:pre_post_processing}. 
For each frame, we compute the intersection over union between the detections and ground-truth boxes of the respective frame, which forms a weighted bipartite graph. Edges with a corresponding intersection over union below $0.5$ are removed. 
Then, we use Hungarian matching to find a maximum-weight matching. Unmatched detections are considered as false positives, while matched detections are assigned the corresponding ground-truth label. Thus, we obtain the trajectories on the input detections using the optimal assignment. Finally, depending on the time threshold of Table \ref{tab:tracking_performance_over_time}, trajectories are synthetically splitted at skip-edges longer than the specified threshold.

\subsection{Ablation study on post-processing methods.}
\label{sec:ablation_postprocessing}

Solving the proposed lifted disjoint paths problem establishes the assignment of input detections to object identities very close to the best possible assignment (Section \ref{sec:exp_long_range_edges}).

To localize tracked objects also in the frames in which the object detector failed to detect them, some trackers apply an additional object detector on these frames based on the available input detections. This can be seen as performing interpolation and extrapolation, if viewed from the perspective of data association in a tracking-by-detection framework, e.g. see \cite{bergmann2019tracking}. 
As a result, improvements can be achieved from extending trajectories  to image areas without input detections by applying of a very accurate object detector.

In order to make our tracking performance comparable with other trackers, we follow this strategy and employ an inter- and extrapolation based on \cite{bergmann2019tracking}. 

During the inter- and extrapolation, output detections (coming from the lifted disjoint paths solver) are preserved.
In particular, the detections are not rejected, reshaped, neither are their labels changed by Tracktor. Instead, we apply Tracktor to recover further locations of an object 
in the frames where detections of the object were missing. The procedure is based on its trajectory obtained from the lifted disjoint paths solver.
Note that our adaption ignores additional, unassigned input detections, whereas the original implementation  \cite{bergmann2019tracking} of Tracktor fuses the detections coming from Tracktor's detector with detections provided by the dataset.

\begin{table}[hbt]
\begin{center}
\begin{tabular}{lcc}
\toprule Method & MOTA & IDF1  \\ \hline
Assignment & $52.8$ &  $64.3$ \\
Assignment (optimal) & $53.4$ &  $66.8$ \\
Assignment+SI & $57.8$ &  $67.6$ \\
Assignment+$\text{SI}^{*}$ & $59.5$ &  $68.9$ \\
Assignment+VI & $59.6$ &  $68.5$ \\
Assignment+VI+VE & $65.7$ &  $71.5$ \\
Assignment+VI+VE+SI & $\mathbf{67.0}$ &  $72.4$ \\
\bottomrule
\end{tabular}
\end{center}
\caption{Ablation study on inter- and extrapolation, evaluated on the MOT17 train set. SI = spatial interpolation only on sequences filmed from a static camera, $\text{SI}^{*}$ = spatial interpolation on all sequences, VI = visual interpolation, VE = visual extrapolation. Assignment and assignment (optimal) denote the results of the lifted disjoint paths problem and the optimal assignment, as reported in Section \ref{sec:exp_long_range_edges} given  $2\mathrm{s}$ time gap. Note that Tracktor's object detector is fine-tuned on MOT17Det. In our experiments, this resulted in bigger improvements on the MOT17 training set than on the test set, compare Table \ref{tab:mot}.
} 
\label{tab:tracking_ablation}
\end{table}

Table \ref{tab:tracking_ablation} reports the influence of employing inter- and extrapolation.
The first two rows repeat values from Table~\ref{tab:tracking_performance_over_time} given the maximal $2\mathrm{s}$ time gap. Since our solver produces nearly optimal data assignemt with respect to the used input detections, further improvements can only be achieved by applying interpolation and extrapolation on the tracks obtained by the solver.

We compare the visual interpolation (VI) as well as visual extrapolation (VE), both using the method of \cite{bergmann2019tracking} with spatial interpolation (SI). For SI, we employ linear interpolation based solely on the geometric bounding box information. 

The interpolation SI is applied only to sequences with a fixed camera in order to guarantee robust approximations.
Still, the improvements by Assignment+SI over the baseline is evident. Especially the MOTA metric, which measures mainly the coverage of objects by detections, improves by about 10\%. We also evaluate spatial interpolation for all sequences ($\text{SI}^{*}$), which improves the tracker further to $59.5$ MOTA and $68.9$ IDF1. However, performing spatial interpolation on sequences with moving cameras can lead to error propagation. Thus, our final tracker Lif\_T relies on the more robust visual interpolation and employs spatial interpolation only on sequences filmed from a static camera.

On the contrary, the  visual interpolation based on \cite{bergmann2019tracking} can be applied robustly to all sequences, but only in situations where the object is visible.
Accordingly, the method Assignment+VI further improves over the baseline, as it is applied to more frames.

Recovering the position of tracked objects also outside of the time range of its computed trajectory (Assignment+VI+VE) further helps to improve the tracking accuracy, enhancing MOTA by about 20\% and IDF1 by about 10\% IDF1, as VE extends computed trajectory thereby achieving longer identity consistencies.

Finally, we employ spatial interpolation on the remaining cases where detections are missing and the objects are fully occluded (Assignment+VI+VE+SI) resulting in a slight improvement over Assignment+VI+VE.

Note that we use the method Assignment+VI+VE+SI to evaluate our tracker on the MOT15, MOT16 and MOT17 test set, as reported in Table \ref{tab:mot}. The impact of the post-processing on the training set using Tracktor seems to be very high. We conjectured this might be due to the fact that Tracktor's object detector is trained on MOT17Det (which are the detections of MOT17), leading to some degree of overfitting. Note that Tracktor is not trained the MOT17 tracking ground truth, so that it is still regarded as a meaningful validation procedure \cite{bergmann2019tracking}.
Therefore, we created another tracker Lif\_TsimInt that uses a simple interpolation, namely only linear interpolation between detections of a trajectory, for all sequences. The tracker thus corresponds to Assignment+$\text{SI}^{*}$.
Comparing Table \ref{tab:tracking_ablation} with Table \ref{tab:mot}, we see that indeed, the impact of the post-processing on the test set is significantly lower. We conclude that while the post-processing improves the tracking performance, the main performance of our tracker is due to our contributions.

Recall that most offline tracking systems obtain trajectories by solving a data association problem, e.g.~\cite{Henschel_2018_CVPR_Workshops,tang2017multiple,ristani2018features}. Our proposed tracker is able to achieve near-optimal results with respect to the input detections. Applying interpolation and extrapolation further improves the results, and makes it conceptually comparable to Tracktor. 
Still, with post-processing on our computed data-association, we improve over Tracktor by 25\%. We argue that solving the data association accurately is important to obtain a final high-quality result after post-processing.

\subsection{Further Details on the Feature Fusion Network.}
\label{sec:fusion_network_details}
We discuss in detail the neural network which fuses the input features, thereby extending Section \ref{sec:cost_learning}. 
\paragraph{Architecture of the fusion network. } 
Considering one assignment hypothesis represented by an edge $e=vw$, the DeepMatching densities $\rho \in [0,1]^{6}$ as well the temporal distance $t$ between the corresponding detections $v$ and $w$ serve as a confidence score for the remaining input features.  
They describe which of the input features is a reliable metric for a given assignment hypothesis, but they are not giving any information about the correctness of the assignment hypothesis. We transform the density features non-linearly and denote them together with the temporal distance as control features $\mathcal{C}(e):=(\log(\rho),t) \in \R^{6} \times [0,2]$. The remaining features described in Section \ref{sec:cost_learning} are denoted as  $\mathcal{F}(e) \in [0,1]^{n}$.

One plausible architecture is to use a convex combination of the input features, such that the coefficients depend on the control features. To this end, let $\alpha_{i}(\mathcal{C}(e),W_{\alpha_{i}})$ for $i=1,\cdots,n$ denote a neural network with the control features as input and $W_{\alpha_{i}}$ as learnable weights. Further, let $\beta_i(\mathcal{F}(e)_{i},W_{\beta_{i}})$ for $i=1,\cdots,n$ be a neural network applied to $i$-th feature of $\mathcal{F}(e)$, with learnable weights $W_{\beta_{i}}$. 

The input features and control features can then be fused via

\begin{equation}
\label{eqn:convex_fusion}
\sum_{i=1}^{n} \alpha_{i}(\mathcal{C}(e),W_{\alpha_{i}}) \beta_i(\mathcal{F}(e)_{i},W_{\beta_{i}}),     
\end{equation}{}
such that
\begin{equation}
\sum_{i=1}^{n} \alpha_{i}(\mathcal{C}(e),W_{\alpha_{i}}) = 1.
\end{equation}
To ensure stable training, \eqref{eqn:convex_fusion} should be applied to a sigmoid function and trained using binary cross-entropy loss.

Nonetheless, our tracker implementation employs neural network based mainly on a combination of relu units and fully connected layers, which performed slightly better, still sharing the idea of seperating the input into  control features and input features.
The detailed architecture is depicted in Figure \ref{fig:fusion_network}.

\begin{figure}
    \centering
    \includegraphics{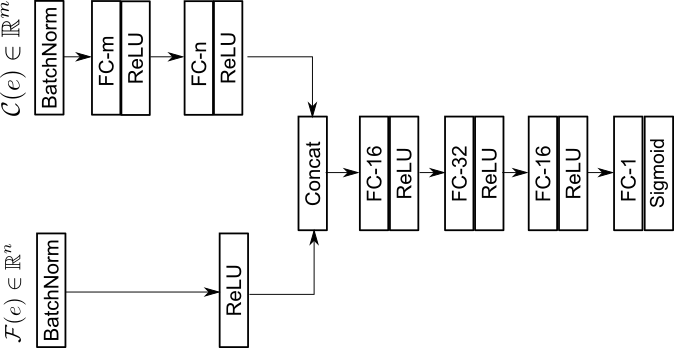}
    \caption{The architecture of the edge classifier used in Lif\_T. FC-$i$ denotes a fully-connected layer with $i$ nodes in as outputs. 
    Using a concatenation with subsequent fully connected layer, $m$ control features and $n$ input features are fused. }
    \label{fig:fusion_network}
\end{figure}{}

\paragraph{Training details.}
Training of the neural network is performed directly on the (preprocessed) input detections.  Labels are retrieved by assigning each detection to the best fitting ground-truth bounding box. Detections with ambiguous assignments are ignored within the training phase. 

In order to train the edge classifier, special care has to be taken as the training set is highly imbalanced. The number of edges which correspond to true negatives (pairs of detections which do not belong to the same person) clearly dominates the number of true positive edges (pairs of detections belonging to the same person). 

To address this issue, the network is trained on a randomly sampled subset of all possible edges, such that the ratio of true positive edges and true negative edges per time distance between the end nodes of the edges remains fixed.  The maximal temporal distance of an edge is set to $2$ seconds, allowing to recover persons even after long occlusions. 

The weights of the fusion network are optimized according to the binary cross-entropy loss.
We employ stochastic gradient descent with the learning rate set to $10^{-2}$ and Nesterov momentum set to $0.9$, for a total of 10 epochs. Training and inference is performed using Pytorch 1.3 on a Nvidia  RTX 2080 Ti. 

\paragraph{Accuracy of the fusion network.}
The performance of a tracking system depends highly on the accuracy of the edge classifier (and the corresponding edge weights).

Therefore, we report our evaluation of the edge classifier on all training sequences of the filtered MOT17 train set in Table \ref{tab:edge_classifier}.
Together with Table \ref{tab:mot_eval_per_sequence} and Table \ref{tab:mot}, it shows that improvements in the tracking features directly correlate to high quality tracking results thanks to the proposed solver. While Table \ref{tab:edge_classifier} shows very good performance of the edge classifier, a powerful graph model and solver is still crucial to obtain high quality tracking results. Even small errors (we observed 5\% maximal error) in the edge classifier can cause many errors in the tracking results if an unsuitable procedure is used. 
Also note that for training the edge classifier, detections with ambiguous assignment to the ground truth boxes were ignored. So, these potentially difficult cases are excluded int the evaluation of the edge classifier.
Especially the interpolation and extrapolation is prone to error propagation, once a single identity switch has been created, which heavily affects, among others,  the IDF1 score. Our lifted disjoint paths formulation can be advantageous, since lifted edges aggregate multiple edge classifiers which can correct individual wrong classifications of single edges.

\begin{table}[]
    \centering
    \resizebox{\columnwidth}{!}{
    \begin{tabular}{l  c c c c}
    \toprule
        Sequence & Acc $\uparrow$ & Prec $\uparrow$  & TPR $\uparrow$ & TNR $\uparrow$ \\ 
        \toprule
        MOT17-02-DPM & $1.00$ & $0.99$ & $1.00$ & $1.00$ \\
        MOT17-04-DPM & $1.00$ &  $0.98$ & $0.99$ & $1.00$ \\
        MOT17-05-DPM & $0.95$ & $0.95$ & $1.00$ & $0.99$\\
        MOT17-09-DPM & $1.00$ & $0.98$ & $0.98$ & $1.00$ \\
        MOT17-10-DPM & $1.00$ & $0.99$ & $0.99$ & $1.00$ \\
        MOT17-11-DPM & $1.00$ & $1.00$ & $0.99$ & $1.00$ \\
        MOT17-13-DPM & $0.99$ & $0.97$ & $0.96$ & $1.00$ \\
        
        MOT17-02-SDP & $1.00$ & $0.96$ & $1.00$ & $1.00$ \\
        MOT17-04-SDP & $1.00$ & $0.98$ & $0.98$ & $1.00$ \\
        MOT17-05-SDP & $0.99$ & $0.92$ & $1.00$ & $0.98$ \\
        MOT17-09-SDP & $0.97$ & $0.81$ & $0.99$ & $0.97$ \\
        MOT17-10-SDP & $0.99$ & $0.94$ & $0.97$ & $1.00$ \\
        MOT17-11-SDP & $1.00$ & $0.99$ & $0.99$ & $1.00$ \\
        MOT17-13-SDP & $0.99$ & $0.90$ & $0.96$ & $0.99$ \\
        
        MOT17-02-FRCNN & $1.00$ & $0.98$ & $1.00$ & $1.00$ \\
        MOT17-04-FRCNN & $1.00$ & $0.97$ & $0.99$ & $1.00$ \\
        MOT17-05-FRCNN & $0.99$ & $0.94$ & $1.00$ & $1.00$ \\
        MOT17-09-FRCNN & $0.99$ & $0.97$ & $0.98$ & $1.00$ \\
        MOT17-10-FRCNN & $0.99$ & $0.95$ & $0.98$ & $1.00$ \\
        MOT17-11-FRCNN & $1.00$ & $0.99$ & $0.99$ & $1.00$ \\
        MOT17-13-FRCNN & $0.99$ & $0.90$ & $0.95$ & $0.99$ \\
        \bottomrule
    \end{tabular}}
    \caption{Performance metrics on the edge classifier. The performance is measured in terms of the accuracy (Acc), precision (Prec), true positive rate (TPR) and true negative rate (TNR). The arrows indicate that higher metric values are better.  }
    \label{tab:edge_classifier}
\end{table}

\subsection{Extended Quantitative  Results}
\label{sec:results_all_sequences}
We provide additional evaluations on our tracking system as well as on the lifted disjoint paths solver.

\paragraph{Detailed tracking evaluations.}
We provide the evaluations of the MOT15, MOT16, MOT17 test sets as well as the MOT17 train set per sequence in Table \ref{tab:mot_eval_per_sequence}. 
In addition, the table contains the solver time (STime) in seconds, needed to solve the corresponding lifted disjoint paths problem.

\begin{table*}
\center
\tabcolsep=0.09cm

    \resizebox{1.3\columnwidth}{!}{
    \begin{tabular}{c l c c c c c c c c c}
     \toprule
     & Sequence  & MOTA$\uparrow$  & IDF1$\uparrow$  & MT$\uparrow$  & ML$\downarrow$  & FP$\downarrow$ & FN$\downarrow$ & IDS$\downarrow$ & Frag$\downarrow$ &STime$\downarrow$ \\   
     \toprule
     \parbox[t]{3mm}{\multirow{6}{*}{\rotatebox[origin=c]{90}{MOT17-Train}}}
 & MOT17-02-DPM & $40.5$ & $50.3$ & $13$ & $29$ & $19$ & $11017$ & $26$ & $23$ & $127$ \\ 
 & MOT17-04-DPM & $69.9$ & $73.9$ & $41$ & $22$ & $298$ & $13986$ & $38$ & $41$ & $1521$ \\ 
 & MOT17-05-DPM & $58.2$ & $67.0$ & $31$ & $40$ & $40$ & $2824$ & $27$ & $65$ & $36$ \\ 
 & MOT17-09-DPM & $72.9$ & $71.6$ & $14$ & $1$ & $58$ & $1370$ & $15$ & $7$ & $59$ \\ 
 & MOT17-10-DPM & $67.4$ & $70.2$ & $26$ & $8$ & $106$ & $4043$ & $39$ & $82$ & $173$ \\ 
 & MOT17-11-DPM & $67.3$ & $73.9$ & $24$ & $26$ & $55$ & $3017$ & $11$ & $28$ & $115$ \\ 
 & MOT17-13-DPM & $63.6$ & $67.2$ & $45$ & $36$ & $64$ & $4127$ & $43$ & $48$ & $59$ \\ 
 & MOT17-02-FRCNN & $47.4$ & $57.2$ & $15$ & $22$ & $89$ & $9656$ & $26$ & $27$ & $229$ \\ 
 & MOT17-04-FRCNN & $67.5$ & $74.1$ & $38$ & $21$ & $98$ & $15310$ & $29$ & $13$ & $1535$ \\ 
 & MOT17-05-FRCNN & $60.2$ & $68.9$ & $35$ & $36$ & $73$ & $2651$ & $30$ & $62$ & $92$ \\ 
 & MOT17-09-FRCNN & $71.5$ & $72.9$ & $14$ & $1$ & $54$ & $1451$ & $10$ & $7$ & $51$ \\ 
 & MOT17-10-FRCNN & $73.2$ & $76.2$ & $33$ & $2$ & $270$ & $3096$ & $73$ & $145$ & $398$ \\ 
 & MOT17-11-FRCNN & $73.1$ & $78.8$ & $32$ & $18$ & $82$ & $2436$ & $18$ & $27$ & $133$ \\ 
 & MOT17-13-FRCNN & $77.1$ & $75.8$ & $68$ & $10$ & $203$ & $2394$ & $73$ & $89$ & $388$ \\ 
 & MOT17-02-SDP & $55.0$ & $61.3$ & $16$ & $16$ & $65$ & $8236$ & $52$ & $50$ & $586$ \\ 
 & MOT17-04-SDP & $77.7$ & $81.8$ & $46$ & $13$ & $243$ & $10296$ & $49$ & $66$ & $4133$ \\ 
 & MOT17-05-SDP & $64.0$ & $69.5$ & $41$ & $22$ & $105$ & $2351$ & $33$ & $84$ & $80$ \\ 
 & MOT17-09-SDP & $73.0$ & $73.0$ & $14$ & $1$ & $69$ & $1356$ & $12$ & $12$ & $127$ \\ 
 & MOT17-10-SDP & $75.0$ & $78.6$ & $35$ & $2$ & $349$ & $2759$ & $105$ & $160$ & $756$ \\ 
 & MOT17-11-SDP & $74.4$ & $78.4$ & $36$ & $14$ & $115$ & $2277$ & $27$ & $36$ & $198$ \\ 
 & MOT17-13-SDP & $70.8$ & $71.4$ & $62$ & $24$ & $200$ & $3150$ & $55$ & $81$ & $364$ \\ 
  & MOT17-Train & $67.0$ & $72.4$ & $679$ & $364$ & $2655$ & $107803$ &  $791$ & $1153$ & $11430$ \\

     \midrule
     \parbox[t]{3mm}{\multirow{6}{*}{\rotatebox[origin=c]{90}{MOT17-Test}}} &
 MOT17-01-DPM &  $48.3$ &$ 58.1 $& $8 $ & $11$  & $68$ & $3258$ & $10$ & $19$  & $38$ \\ 
& MOT17-03-DPM  &  $73.3$ & $70.1$ & $82$ & $17$ & $3560$ & $24276$ & $160$ & $256$  & $24311$ \\ 
& MOT17-06-DPM  &  $58.1$ & $64.7$ & $61$ & $77$ & $178$ & $4728$ & $28$ & $155$  & $113$ \\ 
& MOT17-07-DPM  &  $44.4$ & $52.3$ & $7$ & $21$ & $155$ & $9176$ & $60$ & $209$  & $297$ \\ 
& MOT17-08-DPM  &  $34.7$ & $47.4$ & $18$ & $37$ & $254$ & $13507$ & $32$ & $44$  & $146$ \\ 
& MOT17-12-DPM  &  $48.3$ & $62.3$ & $18$ & $41 $ & $35$ & $4437$ & $11$ & $52$  & $68$ \\ 
& MOT17-14-DPM  &  $36.1$ & $48.8$ & $12$ & $77$ & $268$ & $11449$ & $91$ & $239$  & $323$ \\ 
& MOT17-01-FRCNN  &  $47.7$ & $58.1$ & $8$ & $10$ & $246$ & $3119$ & $7$ & $24$  & $79$ \\ 
& MOT17-03-FRCNN  &  $72.2$ & $71.8$ & $71$ & $17$ & $2664$ & $26277$ & $124$ & $250$  & $11678$ \\ 
& MOT17-06-FRCNN  &  $60.4$ & $63.7$ & $68$ & $61$ & $279$ & $4358$ & $32$ & $207$  & $203$ \\ 
& MOT17-07-FRCNN  &  $44.0$ & $54.9$ & $8$ & $20$ & $279$ & $9110$ & $63$ & $227$  & $281$ \\ 
& MOT17-08-FRCNN  &  $31.9$ & $43.3$ & $17$ & $37$ & $383$ & $13973$ & $35$ & $59$  & $130$ \\ 
& MOT17-12-FRCNN  &  $47.3$ & $58.0$ & $16$ & $43$ & $37$ & $4521$ & $11$ & $34$  & $84$ \\ 
& MOT17-14-FRCNN  &  $36.2$ & $49.0$ & $16$ & $72$ & $629$ & $11061$ & $108$ & $358$  & $359$ \\ 
& MOT17-01-SDP  &  $47.8$ & $57.8$ & $9$ & $10$ & $346$ & $3008$ & $10$ & $31$  & $95$ \\ 
& MOT17-03-SDP  &  $78.2$ & $77.3$ & $92$ & $13 $ & $3778$ & $18879$ & $132$ & $323$  & $16219$ \\ 
& MOT17-06-SDP  &  $60.3$ & $65.1$ & $67$ & $64$ & $305$ & $4345$ & $33$ & $217$  & $144$ \\ 
& MOT17-07-SDP  &  $45.8$ & $55.0$ & $8$ & $18 $ & $285$ & $8793$ & $71$ & $280$  & $483$ \\ 
& MOT17-08-SDP  &  $34.8$ & $47.7$ & $18$ & $34 $ & $429$ & $13288$ & $48$ & $69$  & $202$ \\ 
& MOT17-12-SDP  &  $47.3$ & $60.7$ & $18$ & $42 $ & $158$ & $4394$ & $14$ & $53$  & $85$ \\ 
& MOT17-14-SDP  &  $38.3$ & $51.4$ & $15$ & $69 $ & $630$ & $10662$ & $109$ & $370$  & $376$ \\

     \midrule
     
     \parbox[t]{3mm}{\multirow{6}{*}{\rotatebox[origin=c]{90}{MOT16}}} &
     MOT16-01 &  $48.3$ &$ 58.2 $& $8 $ & $10$  & $78$ & $3217$ & $10 $& $19$  & $38$ \\ 
& MOT16-03  &  $73.0$ & $69.9$ & $80$ & $17$ & $3732$ & $24329$ & $159$ & $310$  & $24311$ \\ 
& MOT16-06  &  $58.2$ & $64.7$ & $62$ & $77$ & $249$ & $4548$ & $29$ & $159$  & $113$ \\ 
& MOT16-07  &  $45.6$ & $53.4$ & $7$ & $16 $ & $189$ & $8637$ & $57$ & $212$  & $297$ \\ 
& MOT16-08  &  $43.4$ & $55.7$ & $18$ & $24 $ & $284$ & $9149$ & $32$ & $44$  & $146$ \\ 
& MOT16-12  &  $50.2$ & $64.0$ & $18$ & $37 $ & $44$ & $4072$ & $11$ & $51$  & $68$ \\ 
& MOT16-14  &  $36.1$ & $48.8$ & $12$ & $77$ & $268$ & $11449$ & $91$ & $239$  & $323$ \\ 
     \midrule
     
     \parbox[t]{3mm}{\multirow{6}{*}{\rotatebox[origin=c]{90}{2D MOT15}}} &
     ADL-Rundle-1 &  $39.6$ &$ 60.8 $& $13 $ & $2$  & $2277$ & $3303$ & $44 $& $175$  & $325$ \\ 
     & ADL-Rundle-3  &  $59.2$ & $69.9$ & $23$ & $7 $ & $902$ & $3217$ & $29$ & $42$  & $153$ \\ 
     & AVG-TownCentre  &  $61.8$ & $67.3$ & $96$ & $33 $ & $417$ & $2217$ & $99$ & $213$  & $20$ \\ 
     & ETH-Crossing  &  $57.6$ & $69.3$ & $7$ & $9 $ & $35$ & $387$ & $3$ & $18$  & $2$ \\ 
     & ETH-Jelmoli  &  $51.4$ & $67.1$ & $18$ & $14 $ & $520$ & $701$ & $12$ & $44$  & $20$ \\ 
& ETH-Linthescher  &  $53.7$ & $62.2$ & $42$ & $98 $ & $318$ & $3795$ & $21$ & $95$  & $11$ \\ 
& KITTI-16  &  $36.2$ & $32.7$ & $5$ & $1 $ & $456$ & $521$ & $108$ & $60$  & $57$ \\ 
& KITTI-19 &  $43.3$ & $49.4$ & $11$ & $17 $ & $467$ & $2315$ & $249$ & $142$  & $135$ \\ 
& PETS09-S2L2  &  $56.9$ & $43.6$ & $9$ & $2$ & $476$ & $3531$ & $152$ & $225$  & $180$ \\ 
& TUD-Crossing  &  $88.0$ & $90.9$ & $11$ & $0 $ & $64$ & $62$ & $6$ & $13$  & $13$ \\ 
& Venice-1  &  $45.8$ & $62.1$ & $9$ & $3$ & $905$ & $1561$ & $7$ & $20$  & $30$ \\ 
     \bottomrule
    \end{tabular}}

\caption{We provide the results of our tracker Lif\_T, evaluated per sequence. In addition, we provide the time necessary to solve the corresponding lifted disjoint path problem instance (STime), in seconds. 
Arrows indicate whether low or high metric values are better. Tracking results on the test sets were evaluated by the MOTChallenge server \url{https://www.motchallenge.net}}
\vspace{-0.2cm}
\label{tab:mot_eval_per_sequence}

\end{table*}

\end{document}